\documentclass[sigconf]{aamas}

\usepackage{balance} %

\setcopyright{ifaamas}
\acmConference[AAMAS '23]{Proc.\@ of the 22nd International Conference
on Autonomous Agents and Multiagent Systems (AAMAS 2023)}{May 29 -- June 2, 2023}
{London, United Kingdom}{A.~Ricci, W.~Yeoh, N.~Agmon, B.~An (eds.)}
\copyrightyear{2023}
\acmYear{2023}
\acmDOI{}
\acmPrice{}
\acmISBN{}

\acmSubmissionID{41}

\title{D-Shape: Demonstration-Shaped Reinforcement Learning via Goal Conditioning}

\author{Caroline Wang}
\affiliation{
  \institution{The University of Texas at Austin}
  \city{Austin}
  \country{Texas, United States}}
\email{caroline.l.wang@utexas.edu}

\author{Garrett Warnell}
\affiliation{
  \institution{Army Research Laboratory and \\ The University of Texas at Austin}
  \city{Austin}
  \country{Texas, United States}}
\email{garrett.a.warnell.civ@army.mil}

\author{Peter Stone}
\affiliation{
  \institution{The University of Texas at Austin and Sony AI}
  \city{Austin}
  \country{Texas, United States}}
\email{pstone@cs.utexas.edu}

\newcommand{\BibTeX}{\rm B\kern-.05em{\sc i\kern-.025em b}\kern-.08em\TeX}
\usepackage{multicol}
\usepackage{booktabs}
\usepackage{amsmath, amsfonts, amsthm}
\usepackage{dsfont}
\usepackage{mathtools}
\usepackage[normalem]{ulem} 
\usepackage{algorithm}
\usepackage{algorithmic}

\def\R{\mathds R}
\def\E{\mathds E}
\DeclareMathOperator*{\argmax}{arg\,max}

\newtheorem{theorem}{Theorem}

\usepackage{graphicx}
\usepackage{caption}
\usepackage{subcaption}


\begin{document}

\pagestyle{fancy}
\fancyhead{}

\begin{abstract}
While combining imitation learning (IL) and reinforcement learning (RL) is a promising way to address poor sample efficiency in autonomous behavior acquisition, methods that do so typically assume that the requisite behavior demonstrations are provided by an expert that behaves optimally with respect to a task reward. If, however, suboptimal demonstrations are provided, a fundamental challenge appears in that the demonstration-matching objective of IL conflicts with the return-maximization objective of RL. This paper introduces D-Shape, a new method for combining IL and RL that uses ideas from reward shaping and goal-conditioned RL to resolve the above conflict. D-Shape allows learning from suboptimal demonstrations while retaining the ability to find the optimal policy with respect to the task reward. We experimentally validate D-Shape in sparse-reward gridworld domains, showing that it both improves over RL in terms of sample efficiency and converges consistently to the optimal policy in the presence of suboptimal demonstrations.
\end{abstract}

\maketitle 

\section{Introduction and Background} 
A longstanding goal of artificial intelligence is enabling machines to learn new behaviors. Towards this goal, the research community has proposed both imitation learning (IL) and reinforcement learning (RL). In IL, the agent is given access to a set of state-action expert demonstrations and its goal is either to mimic the expert's behavior, or infer the expert's reward function and maximize the inferred reward. In RL, the agent is provided with a reward signal, and its goal is to maximize the long-term discounted reward. While RL algorithms can potentially learn optimal behavior with respect to the provided reward signal, in practice, they often suffer from high sample complexity in large state-action spaces, or spaces with sparse reward signals. On the other hand, IL methods are typically more sample-efficient than RL  methods, but require expert demonstration data.
 
It seems natural to consider using techniques from imitation learning and demonstration data to speed up reinforcement learning. However, many IL algorithms implicitly perform divergence minimization with respect to the provided demonstrations, with no notion of an extrinsic task reward \cite{Ghasemipour2019ADM}. When we have access to both demonstration data \textit{and} an extrinsic task reward, we have the opportunity to combine IL and RL techniques, but must carefully consider whether the demonstrated behavior conflicts with the extrinsic task reward --- especially when demonstrated behavior is suboptimal. 
Moreover, standard IL algorithms are only valid in situations when demonstrations contain both state and action information, which is not always the case.

The community has recently made progress in the area of imitation from observation (IfO), which extends IL approaches to situations where  demonstrator action information is unavailable,  difficult to induce, or not appropriate for the task at hand. This last situation may occur if the demonstrator's transition dynamics differ from the learner's---for instance, when the expert is a human and the agent is a robot \cite{liuabishek2017, torabi2019}.
While there has been some work on performing IL or IfO with demonstrations that are suboptimal with respect to an implicit task that the demonstrator seeks to accomplish \cite{Chen2020LearningFS, Brown2019BetterthanDemonstratorIL, Wu2019ImitationLF, Brown2019ExtrapolatingBS}, to date, relatively little work has considered the problem of combining IfO and RL, where the learner's true task is explicitly specified by a reward function \cite{Taylor2011IntegratingRL}.

This paper introduces the D-Shape algorithm, which combines IfO and RL in situations with suboptimal demonstrations.
D-Shape requires only a single, suboptimal, state-only expert demonstration, and treats demonstration states as goals. To ensure that the optimal policy with respect to the task reward is not altered, D-Shape uses potential-based reward shaping to define a goal-reaching reward. We show theoretically that D-Shape preserves optimal policies, and show empirically that D-Shape improves sample efficiency over related approaches with both optimal and suboptimal demonstrations.

\begin{figure}[ht]
\begin{center}
  \includegraphics[scale=0.85]{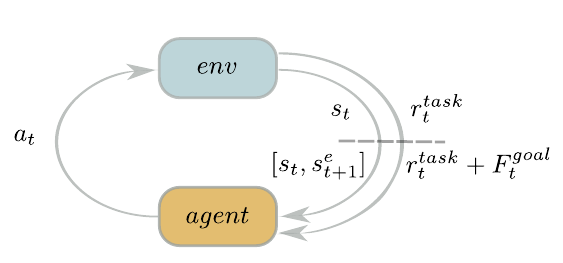}
  \caption{ D-Shape's interaction with the environment. The state $s_t$ and the task reward $r^{task}_t$ come from the environment.  $s_t$ is concatenated with the demonstration state $s^e_{t+1}$ as the goal, and $r^{task}$ is augmented with the potential-based goal reaching function, $F^{goal}_t$}.
  \label{figure:ours_data_gathering}    
\end{center}
\end{figure}

\section{Preliminaries}
\label{section:prelim}
This section introduces our notation, and technical concepts that are key to this work: reinforcement learning with state-only demonstrations, goal-conditioned reinforcement learning, and potential-based reward shaping.
\subsection{Reinforcement Learning with State-Only Demonstrations} 
\label{prelim:rl}
Let $M = (S, A, P, r^{task}, \gamma)$ be a finite-horizon Markov Decision Process (MDP) with horizon $H$, where $S$ and $A$ are the state and action spaces, $P(s' \mid s, a)$ is the transition dynamics of the environment, $r^{task}(s, a, s')$ is a deterministic extrinsic task reward, and $\gamma \in (0,1)$ is the discount factor.
The objective of reinforcement learning is to discover a policy $\pi(\cdot \mid s)$ that maximizes the expected reward induced by $\pi$, $\E_\pi [ \sum_{t=0}^{H-1} \gamma^t r^{task}(s_t, a_t, s_{t+1}) ]$. In this work, we seek to maximize the same objective, but to do so more efficiently by incorporating additional information in the form of a single, state-only demonstration $D^e = \{s_t^e\}_{t=1}^H$, that may be suboptimal with respect to the task reward. The extent to which the demonstration can improve learning efficiency may depend on its degree of suboptimality. However,  incorporating the demonstration into the reinforcement learning procedure should not alter the optimal policy with respect to $r^{task}$, no matter how suboptimal the demonstration is. Prior literature has referred to this desideratum as \textit{policy invariance} \cite{Ng99policyinvariance}.

\subsection{Goal-Conditioned Reinforcement Learning}
\label{prelim:gcrl}

Goal-conditioned reinforcement learning (GCRL) further considers a set of goals $G$ \cite{Schaul2015UniversalVF,  kaelbling93learningto}. While standard RL aims to find policies that can be used to execute a single task, the objective of GCRL is to learn a goal-conditioned policy $\pi(\cdot \mid [s, g])$, where the task is to reach any goal $g \in G$. Typically, $G$ is a predefined set of desirable \textit{states}, and the reward function depends on the goal. A common choice of reward function is the sparse indicator function for when a goal has been reached, $r_t^g = \mathds{1}_{s_t = g}$. 

Since it is challenging for RL algorithms to learn under sparse rewards, \citet{andrychowicz17her} introduced  hindsight experience replay (HER). In the setting considered by \citet{andrychowicz17her}, the goal is set at the beginning of an episode and remains fixed throughout. HER relies on the insight that even if a trajectory fails to reach the given goal $g$, transitions in the trajectory are successful examples of reaching \textit{future} states in the trajectory. More formally, given a transition with goal $g$, $([s_t, g], a_t, r_t^g, [s_{t+1}, g])$, HER samples a set of goals from future states in the episode, $\mathcal{G}$. For all goals $g' \in \mathcal{G}$, HER relabels the original transition to $([s_t, g'], a_t, r_t^{g'}, [s_{t+1}, g'])$ and stores the relabelled transition in a replay buffer. An off-policy RL algorithm is used to learn from the replay buffer. 

In this work, we use demonstration states as goals, allowing the goal to change dynamically throughout the episode, and employ the relabelling technique from HER.

\subsection{Potential-Based Reward Shaping}
\label{prelim:potential}
Define a \textit{potential function} $\phi: S \mapsto \R$. Let $F(s, s') \coloneqq \gamma \phi(s') - \phi(s)$; $F$ is called a \textit{potential-based shaping function}. Consider the MDP $M' = (S, A, P, R' \coloneqq r^{task} + F, \gamma)$, where $r^{task}$ is an extrinsic task reward.  We say $R'$ is a \textit{potential-based reward function}. \citet{Ng99policyinvariance} showed that $F$ being a potential-based shaping function is both a necessary and sufficient condition to guarantee that (near) optimal policies learned in $M'$ are also (near) optimal in $M$ --- that is, policy invariance holds. This work leverages potential-based reward functions as goal-reaching rewards,  to bias the learned policy towards the demonstration trajectory.

\section{Method}
\label{section:method}

We now introduce D-Shape, our approach to improving sample efficiency while leaving the optimal policy according to the task reward unchanged. The training procedure is summarized in Algorithm \ref{alg:ours}.

D-Shape requires only a single, possibly suboptimal, state-only demonstration. We are inspired by model-based IfO methods that rely on inverse dynamics models (IDMs) --- models that, given a current state and a target state, return the action that induces the transition. We observe that an IDM can be viewed as a single-step goal-reaching policy. Although D-Shape does not assume access to an IDM, we hypothesize that providing expert demonstration states as goals to the reinforcement learner might be a useful inductive bias. HER is used to form an implicit curriculum to learn to reach demonstration states. As such, there are three components of D-Shape: state augmentation, reward shaping, and a goal-relabelling procedure. Figure \ref{figure:ours_data_gathering} depicts the state augmentation and reward shaping process of a D-Shape learner as it interacts with the environment. Each component is described below.
 
\begin{algorithm}[htbp]
  \begin{algorithmic}[1]
    \REQUIRE Single, state-only demonstration $D^e \coloneqq \{s^e_t\}_{t=1}^H$
    \STATE Initialize $\theta$ at random
    \WHILE{$\theta$ is not converged}
        \FOR{$t=0:H-1$}
            \STATE Execute $a_t \sim \pi_\theta(\cdot \mid [s_t, s^e_{t+1}])$, observe  $r_t^{task}$, $s_{t+1}$
            \STATE Compute $r_t^{goal} = r_t^{task} + F_t^{goal}$ using Equation (\ref{eq:goal_rew})
            \STATE Store transition 
            \STATE $([s_t, s^e_{t+1}], a_t, r_t^{goal}, [s_{t+1}, s^e_{t+2}])$ in  buffer
        \ENDFOR
        \FOR{$t=1:H-1$}
            \STATE Sample set of consecutive goal states $\mathcal{G}$ uniformly from episode
            \FOR{$(g, g') \in \mathcal{G}$}
                \STATE Recompute $F_t^{goal}$ component of $r_t^{goal}$ using $(g, g')$
                \STATE Relabel transition to 
                \STATE $([s_t, g], a_t, r_t^{goal'}, [s_{t+1}, g'])$
                \STATE Store relabelled transition in replay buffer
            \ENDFOR
        \ENDFOR
        \STATE Update $\theta$ using off-policy RL algorithm
    \ENDWHILE
    \RETURN{$\theta^*$}
  \end{algorithmic}
  \caption{D-Shape}
  \label{alg:ours}
\end{algorithm}

\subsubsection{State augmentation. } 
During training, the policy gathers data with demonstration states as behavior goals: $a_t \sim \pi_\theta(a_t \mid [s_t, s^e_{t+1}])$. The agent observes the next state $s_{t+1}$ and the task reward $r_t^{task}$. The next state $s_{t+1}$ is augmented with the demonstration state $s^e_{t+2}$ as the goal. Note that our method employs dynamic goals, as the goal changes from time step $t$ to time step $t+1$.

\subsubsection{Reward shaping. }
The task reward $r_t^{task}$ is summed with a potential-based shaping function to form a potential-based goal-reaching reward. Define the potential function $\phi([s_t, g_t]) \coloneqq - d(s_t, g_t)$, where $d$ is a distance function, $s_t$ is the state observed by the agent at time $t$, and $g_t$ is the provided goal. Because the goal is defined as part of the state, $\phi(\cdot)$ only depends on the state, as required by the formulation of potential-based reward shaping considered by \citet{Ng99policyinvariance}. The potential-based shaping function is then 

\begin{equation}
    F^{goal}([s_t,g_t], [s_{t+1},g_{t+1}]) \coloneqq \gamma \phi([s_{t+1},g_{t+1}]) - \phi([s_t,g_t]).
    \label{eq:goal_potential}
\end{equation}

The potential-based reward function is
\begin{equation}
    r_t^{goal} \coloneqq r_t^{task} + F_t^{goal}.
    \label{eq:goal_rew}
\end{equation}

Note that $r_t^{goal}$ is a goal-reaching reward that can be recomputed with new goals, suitable for goal relabelling. 
The above procedure results in ``original" transitions of the form $([s_t , s^e_{t+1}], a_t, r_t^{goal}, [s_{t+1}, s^e_{t+2}])$, which are stored in a replay buffer.

\textbf{Goal Relabelling. }
To encourage the policy to reach provided goals, we perform the following goal-relabelling procedure on original transitions using previously achieved states as goals. We adopt a similar technique to HER, with a slight modification to the goal sampling strategy. D-Shape's goal sampling strategy consists of sampling consecutive pairs of achieved states  $(g, g')$ from the current episode. The original transitions are then relabelled to $([s_t, g], a_t, r^{goal'}, [s_{t+1}, g'])$, where the reward is recomputed as $r^{goal'} = r_t^{task} + F_t^{goal}([s_t, g], a_t, [s_{t+1}, g'])$. As the goals are imaginary, even if the goal states change, the task reward $r^{task}_t$ remains unchanged. 

The policy $\pi_\theta$ is updated by applying an off-policy RL algorithm to the original data combined with the relabelled data. In our experiments, we use Q-learning  \cite{qlearningwatkins1989}. At inference time, the policy once more acts with demonstration states as goals, i.e., $a_t \sim \pi_\theta(a_t \mid [s_t, s^e_{t+1}])$.

\section{Consistency with the Task Reward} 
\label{section:theory}

\begin{figure}[th]
\begin{center}
  \includegraphics[width=1.7in,height=1.7in,keepaspectratio]{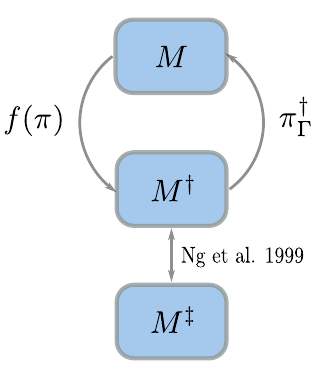}
  \caption{Our theoretical analysis considers D-Shape as the composition of goal relabelling ($M \rightarrow M^\dagger$) and potential-based reward shaping ($M^\dagger \rightarrow M^\ddagger$). D-Shape learns in $M^\dagger$. That optimal policies are preserved from $M^\dagger$ to $M^\ddagger$ follows directly from the policy invariance results of \citet{Ng99policyinvariance}. The theory provided considers a policy $\pi$ in $M$ in $M^\dagger$ via the natural extension, $f(\pi)$, and considers a policy $\pi^\dagger$ in $M^\dagger$ as operating in $M$ via $\Gamma(s)$.}
  \label{figure:mdp_transform}    
\end{center}
\end{figure}

In this section, we prove the claim that D-Shape preserves the optimal policy with respect to the task reward by showing that we can obtain optimal policies on $M$ from an optimal policy learned by D-Shape. The analysis treats D-Shape as the composition of goal relabelling and potential-based reward shaping and formalizes this composition as the MDP transformation, $M \rightarrow M^\dagger \rightarrow M^\ddagger$ (Figure \ref{figure:mdp_transform}). 

Let $M = (S, A, P, r^{task}, \gamma)$ be an MDP with horizon $H$, where $S$ and $A$ are the state and action spaces, $P(s' \mid s, a)$ is the transition dynamics of the environment, $r^{task}(s, a, s')$ is the deterministic extrinsic task reward, and $\gamma \in (0,1)$ is the discount factor. Modify $M$ to $M^\dagger$ as follows: $M^\dagger = ( S \times G, A, P^\dagger, \gamma, r^\dagger)$, where $G$ is a discrete set of goals, and 

\begin{equation*}
    P^\dagger([s', g'] \mid [s, g], a) = P(s' \mid s, a)P(g' \mid [s, g], a),
\end{equation*}
\begin{equation*}
    r^\dagger([s, g], a, [s', g'])  = r^{task}(s, a, s').    
\end{equation*}

We make two independence assumptions in our definition of $P^\dagger$. First, that the random variable $(s' \mid [s, g], a)$ is independent of $(g' \mid [s, g], a)$, allowing us to factorize $P^\dagger([s', g'] \mid [s, g], a) = P(s' \mid [s, g], a) P(g' \mid [s, g], a)$. Second, that $(s' \mid s, g, a)$ is independent of $(g \mid s, a)$, allowing us to rewrite $P(s' \mid [s, g], a) = P(s' \mid s, a)$. In the context of D-Shape, the above assumptions simply mean that a goal in the replay buffer must be independent of all states, goals, and actions other than the previous state, goal, and action. We also require that $r^\dagger$ is independent of goals. We justify that our implementation of D-Shape approximately satisfies these assumptions in the Supplemental Material.

Now define $M^\ddagger = (S \times G, A, P^\dagger, \gamma, r^\dagger + F^{goal})$, where $F^{goal}$ is defined as in Equation \ref{eq:goal_potential}. $M^\ddagger$ is identical to $M^\dagger$, except for the addition of the potential-based shaping function,  $F^{goal}$.

D-Shape learns a goal-conditioned policy $\pi^\ddagger(\cdot \mid [s, g])$ in $M^\ddagger$. To perform inference with the goal-conditioned policy in $M$, we must specify a state-goal mapping $\Gamma : S \mapsto G$. Then $\pi^\ddagger(\cdot \mid [s, \Gamma(s)] )$ can be executed in $M$. Suppose that $\pi^{\ddagger*}(\cdot \mid [s, g])$ is an optimal policy in $M^\ddagger$. Is there a $\Gamma$ such that $\pi^{\ddagger*}(\cdot \mid [s, \Gamma(s)])$ is optimal in $M$? We show next that the answer is positive, and that an arbitrary $\Gamma$ suffices.

That $\pi^{\ddagger*}$ is optimal in  $M^\dagger$ follows from the policy invariance results proven by \citet{Ng99policyinvariance}. By their result, as long as the potential function $F^{goal}$ depends only on the states $([s, g], [s', g'])$ and has the form $F([s, g], [s', g']) = \gamma \phi([s', g']) - \phi([s, g])$, the optimal policy will not be altered by learning under $r^{task} + F^{goal}$.

Thus, the main result reduces to showing that one can obtain an optimal policy in $M$ from an optimal policy in $M^\dagger$(Theorem  \ref{thm:main_result}). The proof of Theorem \ref{thm:main_result} relies on some supporting results, which are stated below and proven in the Supplemental Material. \footnote{Supplemental Material is included in the arXiv version, \texttt{https://arxiv.org/abs/2210.14428.}}

Let us first define a way to map policies in $M$ to $M^\dagger$. Define the policy transformation map, $f(\pi)([s,  g]) = \pi(\cdot \mid s)$. Under the policy transformation map, policies expressible in $M$ can be naturally expressed in $M^\dagger$. 

\textbf{Proposition 1. } Let $\pi(\cdot \mid s)$ be a policy defined in $M$, and let $f(\pi)(\cdot \mid [s,  g]) $ be the extension of $\pi(\cdot \mid s)$ to $M^\dagger$. Then $V_{\pi}(s_t) = V_{f(\pi)}([s_t,  g_t]) \:\: \forall s_t \in S, g_t \in G$, and $t \in \{1,\dots H\}$.

\textbf{Corollary 1.1. } $Q_{f(\pi)}([s_t, g_t], a_t) = Q_{\pi}(s_t, a_t)$ for all $s_t \in S, g_t \in G, a_t \in A,  t \in \{1,\dots H\}$.

\textbf{Proposition 2. } Let $\pi^*(\cdot \mid s)$ be an optimal policy in $M$, and let $f(\pi^*)(\cdot \mid [s,  g])$ be its extension to $M^\dagger$. Then $f(\pi^*)$ is an optimal policy in $M^\dagger$.

Proposition 1 states that the value of $\pi$ in $M$ is equal to the value of $f(\pi)$ in $M^\dagger$ for all $ s \in S, g \in G$, and follows from the observation that the reward function does not depend on goals. Corollary 2 makes the same claim, but for state-action value functions. Proposition 2 states that if a policy $\pi^*$ is optimal in $M$, then $f(\pi^*)$ must also be optimal in $M^\dagger$, and is proven by showing that the policy $f(\pi^*)$ cannot be improved.

\begin{theorem}
\label{thm:main_result}
Let $\pi^{\dagger*}(\cdot \mid [s, g])$ be an optimal policy in $M^\dagger$.  Define an arbitrary state-goal mapping $\Gamma: S \mapsto G_0 \subseteq G$, where $G_0$ is the set of goals in $M^\dagger$ that has positive probability of being reached by any policy. Then $\pi_\Gamma^{\dagger*}(\cdot \mid s) \coloneqq \pi^{\dagger*}(\cdot \mid [s, \Gamma(s)])$ is an optimal policy in $M$.
\end{theorem}

\begin{proof}
We prove the statement by considering the optimal state-action value functions in $M$ and $M^\dagger$. Denote the set of optimal actions at state-goal $[s,  g] \in M^\dagger$ by $A_{s, g} =\argmax_a Q^{\dagger*}([s, g], a) $.

Let $\pi^*(\cdot \mid s)$ be an optimal policy in $M$, and let $f(\pi^*)$ denote the policy induced by $\pi^*$ in $M^\dagger$ as in Proposition 2. By Proposition 2, $f(\pi^*)$ is an optimal policy in $M^\dagger$. Since the optimal value function in $M^\dagger$ is unique, $Q_{f(\pi^*)}([s,g],a) = Q^{\dagger*}([s,g],a)$ for all $s \in S$,  $g \in G$, and $a \in A$. By Corollary 1.1, $Q_{f(\pi^*)}([s,g],a) = Q_{\pi^*}(s,a)$. 
Thus,
\begin{align*}
    A_{s,g} &\coloneqq \argmax_a Q^{\dagger*}([s, g], a) \\
           &= \argmax_a Q_{f(\pi^*)}([s, g], a) \\ 
           &= \argmax_a Q_{\pi^*}(s,a).
\end{align*}
The computation shows that the set of optimal actions in $M^\dagger$ at some $[s,  g] \in S \times G$ is the same as the set of optimal actions in $M$ at $s$ --- that is, the set of optimal actions is independent of the goal $g$. Since the set of actions with positive probability under $\pi^{\dagger*}(\cdot \mid [s,  g])$ is a subset of $A_{s,g}$ for $g\in G_0$, we have $\pi_{\Gamma}^{\dagger *}(\cdot \mid s) := \pi^{\dagger *}(\cdot \mid [s, \Gamma(s)])\subset A_{s, \Gamma(s)} = \argmax_aQ_{\pi^*}(s,a)$ for all $s\in S$. Thus, $\pi^{\dagger*}_\Gamma$ is an optimal policy in $M$.
\end{proof}

The  condition  that $\Gamma(S)  \subseteq G_0 \subseteq G$ avoids the technical issue that at a state $[s,g]$ which has zero probability of being reached, any action is trivially optimal. Thus an optimal policy in $M^\dagger$ at such a state $[s,g]$ might specify an action that is not optimal at $s$ in $M$. The intuition behind the result is that, although the state space and transition probabilities for $M^\dagger$ incorporate a goal distribution, the reward function $r^\dagger$ is goal-independent. We can show that the set of optimal actions at $[s, g] \in S \times G$, operating in $M^\dagger$, is identical to the set of optimal actions at $s$, operating in $M$.\footnote{Indeed, the set of optimal actions at $[s, g]$ is also identical to the set of optimal actions at $[s, g']$ for any $g'$ (operating in $M^\dagger)$, but our proof relies on the set of optimal actions at $s$ (operating in $M$).} Then, since the set of actions considered by $\pi^{\dagger*}$  at $[s, g]$ is optimal for all $[s, g] \in S \times G$, it is certainly optimal at $[s, \Gamma(s)]$.

Theorem \ref{thm:main_result} establishes that it is reasonable to execute $\pi_\Gamma^{\dagger*}$ in $M$ with any $\Gamma(s)$, so long as $\Gamma(S) \subseteq G_0$. In our experiments, we assume that the time step is part of the state and adopt $\Gamma(s_t) = s^e_{t+1}$.\footnote{Appending the time step to the state does not alter optimal policies in an MDP \cite{Pardo2018TimeLI}. This trick is commonly used to ensure that time-limited RL environments maintain the Markov property. Since the time step is defined as part of the state in $M$, the fact that our potential function is time-varying is not a concern. Nevertheless,  note that \citet{devlin12dynamicpbrs} showed that the policy invariance property still holds under time-varying potentials.} Since D-Shape gathers data from the environment using the set of demonstration states as goals, and stores those  transitions in the replay buffer, the technical condition that the goals $G_0$ were seen during learning is satisfied. While our choice of $\Gamma(s)$ may not matter for optimal policies, nevertheless, we use demonstration states as inference-time goals, based on the intuition that demonstration states are useful goals.

\section{Experimental Evaluation}
\label{section:exp_results}
\begin{figure}[t]
\centering
  \includegraphics[width=1.5in]{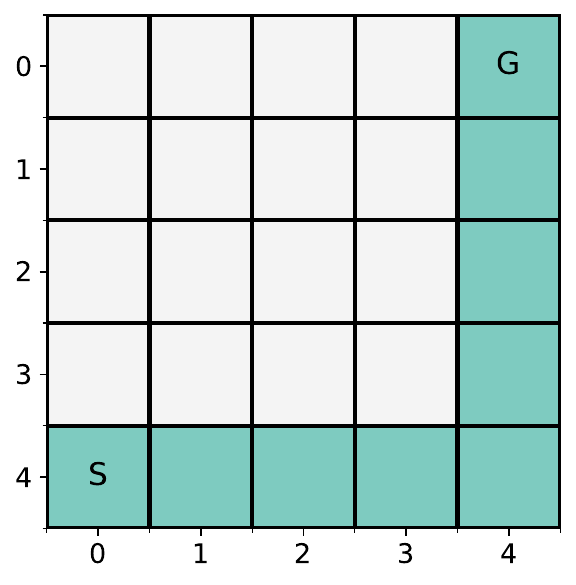}
  \caption{A $5 \times 5$ example of the  gridworld task and the states traversed by the optimal demonstration (shown in blue). The start position is marked by ``S", and the goal position by ``G". The optimal demonstration starts at the ``S" and goes to ``G" along the blue states in the fewest possible number of steps.}
\label{fig:task_fig}
\end{figure}
\begin{figure*}[ht]
\centering
        \includegraphics[width=7in]{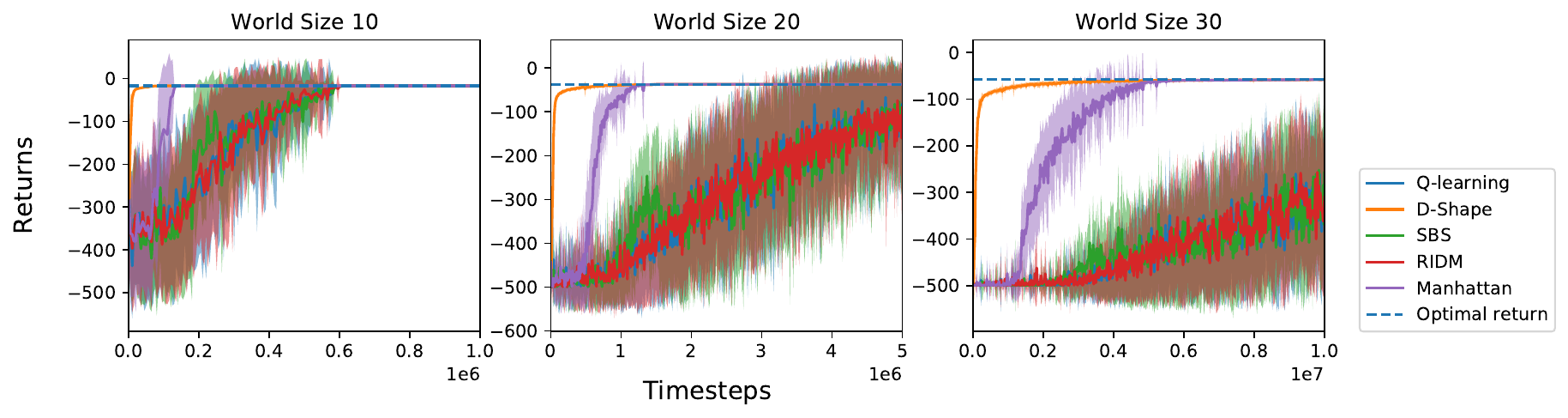}
\caption{Learning curves of D-Shape (our method) using the optimal quality demonstrations, as compared to Q-learning (the base RL method), SBS, RIDM, and Manhattan. D-Shape improves sample efficiency over Q-learning and other baselines by a significant margin over all timesteps.}
\label{fig:core_results}
\end{figure*}

We now empirically study the behavior of D-Shape in gridworld domains. 
We investigate the following questions: 
\begin{itemize}
    \item Does D-Shape improve sample efficiency over RL alone?
    \item How are D-Shape's convergence properties affected by suboptimal demonstrations?
\end{itemize}

Our results show that D-Shape improves sample efficiency over baseline RL methods, and can converge to the optimal policy more rapidly than baseline methods, despite suboptimal demonstrations that do not fully solve the task at hand. 
We also provide the results from an ablation study designed to evaluate the state augmentation and potential-based reward shaping components of D-Shape separately and together with goal-relabelling.

The main experimental procedures are described below; further experimental details are provided in the Supplemental Material.

\subsubsection{Environments.}
The experiments are conducted on gridworlds with side lengths $s \in \{10, 20, 30\}$, where the task is for the agent to navigate from a fixed starting position to some goal position $g$.
Unless otherwise specified, in our experiments, the starting position is one corner of the gridworld, while the goal position is the opposite corner.
The state space consists of the agent's current position, and the action space consists of moving one square up, down, left, or right.
If the agent attempts to move into the boundaries of the gridworld, its position will not change.
The reward function is $r(s) = -\mathds{1}_{s \neq g}$, and episodes terminate either when the agent reaches the goal position or when the episode time limit (500 timesteps) is reached.

These simple gridworld domains were chosen because they possess the key characteristics of having a sparse reward function, and therefore being challenging to explore. 
Further, it is easy to construct optimal and suboptimal demonstrations for this setting, and the state/action spaces are discrete and finite, matching the setting described in the theoretical analysis of D-Shape provided above. 
Experiments were also conducted on a stochastic version of this gridworld, with similar results and findings (see Supplemental Material).

\subsubsection{Baselines.}
For fair comparison, Q-learning \citep{qlearningwatkins1989} is used as the base RL algorithm for all methods and experiments.
D-Shape is compared against RL alone as well as RL+IfO methods that share similar policy invariance guarantees and assumptions.
D-Shape is also compared against RL with a naive RL+IfO reward, where the reward is defined as the sum of the task reward and an imitation reward.
Applying RL with a hybrid task and imitation reward is a common technique to combine RL and IL, but typically requires hyperparameter search and does not provide any robustness guarantees \cite{Zolna2019ReinforcedII, Zhu-RSS-18, Guo2019HybridRL}.
The specific algorithms used as baselines in our experiments are listed below:

\begin{itemize}
    \item \textit{Q-learning } \citep{qlearningwatkins1989}: a classic model-free, off-policy RL method.
    \item \textit{Reinforced Inverse Dynamics Modelling (RIDM)}  \cite{ridm2020}: a model-based RL+IfO algorithm that also assumes only a single demonstration. 
    \item \textit{Similarity Based Shaping (SBS)} \cite{Brys2015ReinforcementLF} : an RL+IL algorithm that incorporates demonstrations through a  potential-based shaping reward. We adapt SBS to the discrete setting, and for state-only demonstrations.
    \item \textit{Q-learning + Manhattan (Manhattan)}: We construct an imitation reward term that consists of the negative Manhattan distance between the state $s_t$ and the demonstration state $s^e_{t+1}$. Q-learning  optimizes for the sum of the imitation reward and the task reward.
\end{itemize}
\begin{figure*}[ht]
\centering
\includegraphics[width=7in]{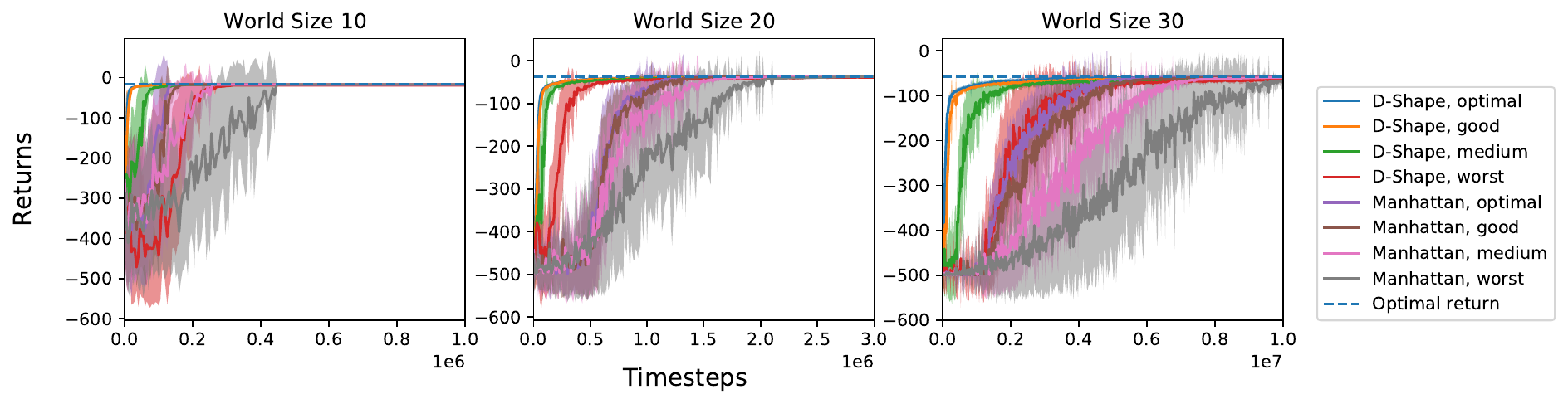}
\caption{Learning curves of D-Shape (our method) and the Manhattan baseline when provided optimal, good, medium, and worst demonstrations. D-Shape converges to the optimal task reward consistently, despite being provided suboptimal demonstrations, and with better sample efficiency than Manhattan.}
\label{fig:vary_demos}
\end{figure*}

\subsubsection{Demonstrations.}
Each algorithm is run with only a single demonstration.
To study the ability of each algorithm to utilize suboptimal demonstrations, the experiments are performed with four demonstrations of varying qualities: optimal, good, medium, and worst.
For the goal-based gridworld tasks considered, any state trajectory from the start position (which is always in the bottom left corner) to the goal position that decreases its distance from the goal by going right or up at each timestep, is an optimal demonstration. Thus, suboptimal demonstrations must consist of trajectories that either (1) go to an incorrect goal, or (2) do not decrease their distance to the goal at some timestep. 

The experiments consider optimal demonstrations---which are manually computed---and suboptimal demonstrations of the first type. More specifically, the optimal demonstrations in the experiments are those that always follow the bottom and right edges of the gridworld; see Figure \ref{fig:task_fig} for an example of the task and style of optimal demonstration used. 
The good, medium, and worst suboptimal demonstrations consist of state trajectories that go to alternative goals that are a Manhattan distance of 2, 4, and 6 steps away from the true goal $g$ for the size $10$ and size $20$ gridworlds. For the size $30$ gridworld, the alternative goals are a Manhattan distance of 4, 8 and 12 steps away from the goal.
Experiments with further types of suboptimal demonstrations (e.g. suboptimal demonstrations of the second type, demonstrations with missing states) may be found in the Supplemental Material.

\subsubsection{Evaluation.} D-Shape and baseline models are trained for a fixed number of time steps. To measure learning speed, the return achieved by each method is evaluated at regular intervals by averaging test returns from the current policy, and is plotted as a learning curve. The shaded region around each learning curve is the standard deviation of the average return. All curves in the following results are computed over 30 independent training runs. The optimal task return is plotted as a horizontal line.

\subsection{D-Shape Improves Sample Efficiency} 

Figure \ref{fig:core_results} compares D-Shape to the baseline methods, where all methods are trained with a demonstration of optimal quality.
It shows that, for all environments, D-Shape was able to improve sample efficiency over Q-learning, SBS, RIDM, and Manhattan.
The closest competing method is Q-learning with the Manhattan distance IfO reward.
The strong performance of this baseline is not surprising, since the Manhattan IfO reward computed with an optimal demonstration is similar to the Manhattan distance of a state to the true goal state---where the latter quantity is actually the value function for this set of tasks \citep{Ng99policyinvariance}. This same Manhattan IfO reward is also the potential function used by D-Shape, yet D-Shape also includes the component of goal relabelling, which may explain why D-Shape's sample efficiency is much higher than that of Manhattan's. The relative contribution of each of D-Shape's components to sample efficiency is further studied in the ablation section.

Figure \ref{fig:state_visitation_dshape} shows the state visitation distribution of a  D-Shape agent on the $20 \times 20$ task, trained with an optimal demonstration analogous to the one shown in Figure \ref{fig:task_fig}. The visitation distribution has highest concentration along the bottom and right edges of the gridworld, which is precisely the path that the demonstration traverses. This confirms that the D-Shape algorithm biases the learning agent to follow the path of the demonstration (rather than following any of the other optimal trajectories for this task).

\begin{figure}[b]
\centering
  \includegraphics[width=2.3in]{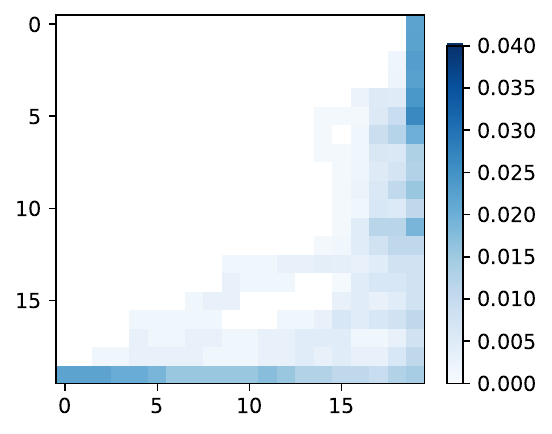}
  \caption{State visitation of D-Shape trained with the optimal demonstration for  the $20 \times 20$ gridworld, where the start position is in the lower leftmost position, and the goal position is in the upper rightmost position. The provided demonstration followed the lower edge to the right, and then went up to the goal state, which is reflected by the state visitation of the learned D-Shape policy.}
\label{fig:state_visitation_dshape}
\end{figure}
\subsection{D-Shape's Convergence and Sample Efficiency is Robust to Demonstration Quality}
\label{subsection:exp_subopt}

The theory provided in Section \ref{section:theory} suggests that D-Shape should converge to the optimal policy even when provided suboptimal demonstrations. This section examines (1) whether the prior theoretical property is exhibited empirically by our implementation of D-Shape, and (2) how the sample efficiency demonstrated by D-Shape in the prior section is affected by suboptimal demonstrations.

More precisely, the theory states that the set of goals D-Shape explores and learns with should not affect D-Shape's ability to learn an optimal policy, as long the goal $g$ at time $t+1$ depends only on the goal, state, and action at time $t$.
We hypothesize that this theoretical guarantee means that D-Shape will still converge to the optimal task reward, no matter the demonstration quality used during training. 

To test the hypothesis, D-Shape is trained with four demonstration qualities: optimal, good, medium, and worst.
In this experiment, D-Shape is compared to the naive RL + IfO baseline, and Q-learning with a Manhattan distance reward.
The same hyperparameters are used for all four demonstration qualities; the hyperparameter tuning procedure is described in the Supplemental Material.

Figure \ref{fig:vary_demos} shows that D-Shape and Q-learning+ Manhattan both converge to the optimal task return for all demonstrations, despite the fact that only the optimal demonstration fully completes the task.
Given the demonstration quality, D-Shape has strictly better sample efficiency and lower learning variance than Q-learning+Manhattan.
While D-Shape's theoretical properties guarantee that it should be able to converge to the optimal policy regardless of the demonstration quality, the Q-learning+Manhattan baseline boasts no such convergence guarantees.
In fact, given a poorly chosen value of the $c$ hyperparameter and a suboptimal demonstration, the learning process of the Q-learning+Manhattan baseline might actually diverge, as the optimal policy has changed (see Figure \ref{fig:subopt_demo_vary_coef}).

\begin{figure}[t]
\centering
  \includegraphics[width=0.45\textwidth]{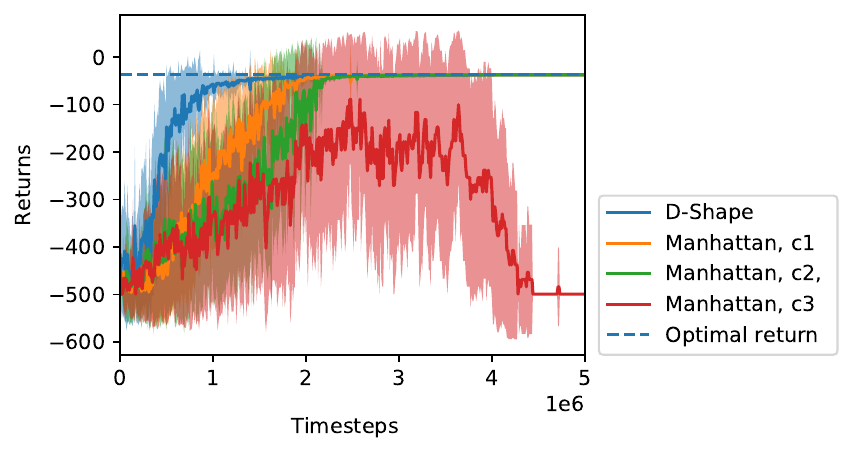}
  \caption{Varying the coefficient on the Manhattan distance IfO reward for the Q-learning + Manhattan distance baseline, for $\vec{c} = [1, 20, 25]$. The learning curve for D-Shape is shown for reference. Note that for $c_3$, Q-learning+Manhattan actually diverges from the optimal task return.}
\label{fig:subopt_demo_vary_coef}
\end{figure}

\subsection{Ablation Study} \label{subsec:ablation}

\begin{figure*}[th]
\centering
  \includegraphics[width=7.5in]{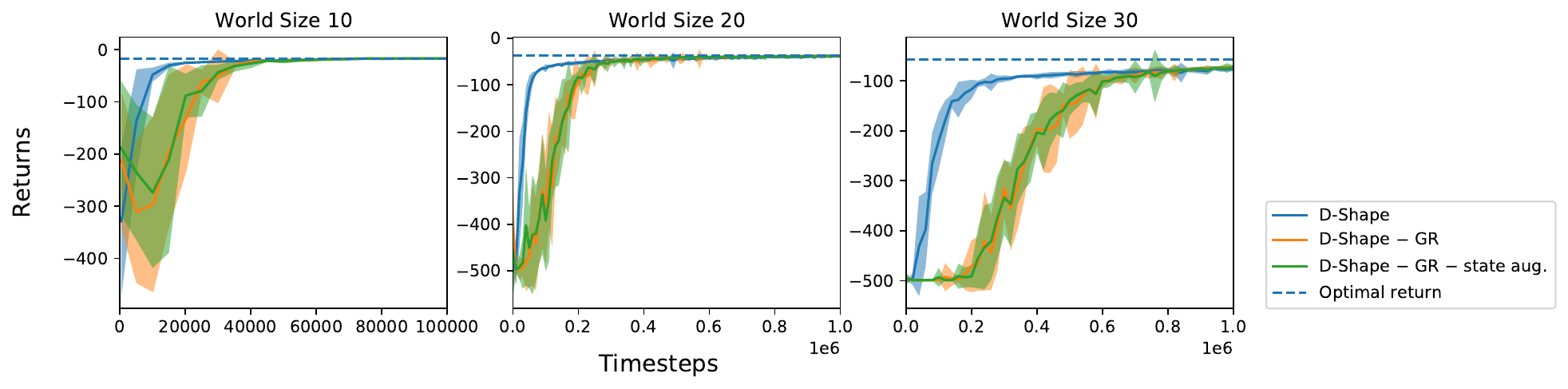}
  \caption{D-Shape compared to two ablations of D-Shape using the optimal demonstration. D-Shape displays better sample-effiency than all ablations. The last ablation ([D-Shape $-$ GR $ -$ shaping]) performed no better than the baseline of Q-learning alone. It is not shown here because the difference between the other ablations would not be visible;  see the Supplemental Material for a version with all ablations.}
\label{fig:ablation}
\end{figure*}

D-Shape draws a connection to GCRL and uses \textit{goal-relabelling} (GR) to unify \textit{state augmentation} and \textit{potential-based reward shaping} from demonstrations (shaping). As such, D-Shape can be thought of as the unification of these three components (GR, state-augmentation, and shaping).

This section investigates three ablations of D-Shape: [D-Shape $-$ GR], [D-Shape $-$ GR $-$ shaping], and [D-Shape $-$ GR $-$ state augmentation].
Note that these are the only three ablations worth studying.
[D-Shape $-$ shaping] is not viable because GR requires a goal-reaching reward; similarly [D-Shape $-$ state augmentation] cannot be studied because goals must be presented to the agent.
[D-Shape $-$ shaping $-$ GR $-$ state augmentation] is equivalent to baseline Q-learning.

The results are shown in Figure \ref{fig:ablation}.
In all domains, D-Shape dominates all three ablations.
Two of the  ablations--- [D-Shape $-$ GR $-$ state augmentation] and [D-Shape $-$ GR] --- perform similarly.
The final ablation, [D-Shape $-$ GR $-$ shaping] performed much more poorly than the others, showing similar sample efficiency to the baseline of Q-learning.
Therefore, this ablation is not shown on Figure \ref{fig:ablation}; see the Supplemental Material to view all ablations together.
The results here imply that for the gridworld domain studied, a large part of D-Shape's improvement in sample efficiency over baseline methods comes from potential-based shaping with demonstrations, and the remaining benefit comes from its GR component.
While state augmentation alone does not seem to produce any gains in sample efficiency, note that it is not possible to perform goal-relabelling without augmenting the state with the goal description.
Thus, we conclude that GR is a crucial component of D-Shape and has a synergistic effect on state augmentation and potential-based reward shaping. 

\section{Related Work} \label{section:related_work}

D-Shape is closely related to methods from the RL+IfO literature and the GCRL literature.
Here, we review work from the RL+IfO literature that is robust to demonstration quality, and work from the GCRL literature that uses demonstrations.

\subsubsection{Combining RL and IL} 

A common approach to combining RL and IL is to define a hybrid reward function that is the sum of the task reward and an imitation reward \cite{Zhu-RSS-18, Guo2019HybridRL}.
Imitation rewards proposed in the literature are derived from similarity measures \cite{Sermanet2018TimeContrastiveNS}, behavior cloning \cite{bain1996, bco2018}, inverse RL \cite{abbeel2004irl, Brown2019BetterthanDemonstratorIL}, or generative adversarial IL \cite{gail2016, gaifo2019, ding19gcil}.
However, if demonstrations are suboptimal, the imitation objective may conflict with the task objective.
The research community has proposed various techniques to remedy this problem.
Simpler solutions consist of tuning the relative weights between the imitation and task reward and/or annealing the imitation reward to zero as training progresses \cite{ding19gcil}.
\citet{Zolna2019ReinforcedII} introduce a method to automatically control the tradeoff between imitation and task objectives during training with Bernoulli random variables.
While each of these methods has demonstrated some experimental success, each introduces a new set of hyperparameters that must be carefully tuned.

Others in the research community have defined hybrid rewards using potential-based reward shaping (PBRS) \cite{Ng99policyinvariance}, and have demonstrated that the policy invariance guarantees of PBRS result in robustness to demonstration quality.
\citet{Brys2015ReinforcementLF} defines a potential function for reward shaping, where the potential function is a similarity measure between agent states and expert states.
Similarly, \citet{Suay2016LearningFD} use inverse RL to recover a linear reward function from the demonstration, and formulate a dynamic potential function from the recovered reward.
\citet{Wu2020ShapingRewards} explore using generative modelling techniques to construct a potential-based shaping function. 
Combining PBRS with goal relabelling has been relatively understudied. 
This work explores using potential-based reward functions as goal-reaching rewards, where the goal is given by demonstration states.

There are several other approaches for combining RL and IL that avoid potential conflict between imitation and task rewards by optimizing only the task reward, and incorporating demonstration knowledge elsewhere in the learning process.
For instance, state augmentation (SA) approaches concatenate the agent state with a representation of the demonstration state(s) \cite{ridm2020, Paine2018OneShotHI}. 
Initialization approaches incorporate demonstration knowledge through pretraining the policy or value functions \cite{Hester2018DeepQF, Taylor2011IntegratingRL}. 
Resetting to demonstration states is another technique that has been shown to be effective for improving the sample efficiency and asymptotic performance of RL \cite{salimans2018LearningMR, Ecoffet2019GoExploreAN,Nair2018OvercomingEI}.
While the aforementioned techniques have demonstrated success with optimal demonstrations, they have not been explicitly evaluated for their robustness to suboptimal demonstrations. 
This work uses SA as a component of D-Shape, and our theoretical results directly apply to SA, showing that SA should not modify the optimal policy with respect to the task reward. 
Therefore, this work provides theoretical grounding for the intuition that SA is robust to demonstration quality.

\subsubsection{Accelerating GCRL with Demonstrations.}
There is some work in the GCRL literature that investigates using demonstrations to speed up learning in classic GCRL tasks (e.g., maze tasks, target-reaching tasks).
For example, \citet{Nair2018OvercomingEI} leverage state-action demonstrations in a GCRL algorithm through a demonstration buffer, an auxiliary behavioral cloning loss with a Q-filter, and resetting to demonstration states.
\citet{ding19gcil} incorporate demonstrations into a GCRL algorithm by adding an imitation reward derived from generative adversarial IL to a sparse goal-reaching reward.
Plan-based reward shaping methods consider designing potential-based shaping rewards when given access to an approximate plan or demonstration\citep{schubert21planbased, Brys2015ReinforcementLF}. 
Most similar to this work, \citet{Paul2019LearningFT} learn subgoals from a set of demonstration trajectories and define a potential-based shaping function from the subgoals.
Unlike \citet{Paul2019LearningFT}, we study goal-relabelling, demonstrate the efficacy of D-Shape using only a single expert demonstration, and examine using demonstration states directly as goals in a time-aligned manner.

\section{Conclusion}
\label{section:conclusion}

In this work, we investigated using techniques from goal-conditioned reinforcement learning to combine potential-based reward shaping and state augmentation, and introduced D-Shape, a novel RL+IfO algorithm.
We showed theoretically that an optimal policy learned using D-Shape can be executed optimally with an arbitrary sequence of goals, and demonstrated empirically that D-Shape improves sample efficiency and is robust to suboptimal demonstrations in gridworld domains.
In our experiments, D-Shape was able to (1) improve sample efficiency over all baselines, and (2) consistently converge to the optimal policy with better sample efficiency than baselines, when provided with suboptimal demonstrations.

This paper evaluates D-Shape based on a sample efficiency criterion in discrete, goal-based gridworlds.
Future work could assess D-Shape's ability to generalize to new demonstrations on more complex gridworld tasks, such as those with obstacles.
Another direction of future work is extending D-Shape to continuous state-action spaces. 
In this work, we augment agent states with demonstration states as goals, where the demonstration states have the same representation as the agent states, but learned goal representations may be necessary to extend D-Shape to continuous state-action spaces.
A limitation of D-Shape is that if the provided potential function does not improve the reinforcement learner's performance, then D-Shape also will not improve performance. 
Thus, learning potential functions—an active research area \citep{Grzes2010OnlineLO, KhansariZadeh2017LearningPF}---is an important future direction.
Finally, this work investigates D-Shape with a single demonstration only, but there is a rich set of techniques from GCRL that could potentially be used to incorporate multiple demonstrations into D-Shape.

\begin{acks}
This work has taken place in the Learning Agents Research
Group (LARG) at the Artificial Intelligence Laboratory, The University
of Texas at Austin.  LARG research is supported in part by the
National Science Foundation (CPS-1739964, IIS-1724157, FAIN-2019844),
the Office of Naval Research (N00014-18-2243), Army Research Office
(W911NF-19-2-0333), DARPA, Lockheed Martin, General Motors, Bosch, and
Good Systems, a research grand challenge at the University of Texas at
Austin.  The views and conclusions contained in this document are
those of the authors alone.  Peter Stone serves as the Executive
Director of Sony AI America and receives financial compensation for
this work.  The terms of this arrangement have been reviewed and
approved by the University of Texas at Austin in accordance with its
policy on objectivity in research.
\end{acks}

\bibliographystyle{ACM-Reference-Format} 
\bibliography{references}


\begin{thebibliography}{40}


\ifx \showCODEN    \undefined \def \showCODEN     #1{\unskip}     \fi
\ifx \showDOI      \undefined \def \showDOI       #1{#1}\fi
\ifx \showISBNx    \undefined \def \showISBNx     #1{\unskip}     \fi
\ifx \showISBNxiii \undefined \def \showISBNxiii  #1{\unskip}     \fi
\ifx \showISSN     \undefined \def \showISSN      #1{\unskip}     \fi
\ifx \showLCCN     \undefined \def \showLCCN      #1{\unskip}     \fi
\ifx \shownote     \undefined \def \shownote      #1{#1}          \fi
\ifx \showarticletitle \undefined \def \showarticletitle #1{#1}   \fi
\ifx \showURL      \undefined \def \showURL       {\relax}        \fi
\providecommand\bibfield[2]{#2}
\providecommand\bibinfo[2]{#2}
\providecommand\natexlab[1]{#1}
\providecommand\showeprint[2][]{arXiv:#2}

\bibitem[\protect\citeauthoryear{Abbeel and Ng}{Abbeel and Ng}{2004}]%
        {abbeel2004irl}
\bibfield{author}{\bibinfo{person}{Pieter Abbeel} {and}
  \bibinfo{person}{Andrew~Y. Ng}.} \bibinfo{year}{2004}\natexlab{}.
\newblock \showarticletitle{Apprenticeship Learning via Inverse Reinforcement
  Learning}. In \bibinfo{booktitle}{\emph{ICML}}. \bibinfo{publisher}{PMLR}.
\newblock


\bibitem[\protect\citeauthoryear{Andrychowicz, Wolski, Ray, Schneider, Fong,
  Welinder, McGrew, Tobin, Pieter~Abbeel, and Zaremba}{Andrychowicz
  et~al\mbox{.}}{2017}]%
        {andrychowicz17her}
\bibfield{author}{\bibinfo{person}{Marcin Andrychowicz}, \bibinfo{person}{Filip
  Wolski}, \bibinfo{person}{Alex Ray}, \bibinfo{person}{Jonas Schneider},
  \bibinfo{person}{Rachel Fong}, \bibinfo{person}{Peter Welinder},
  \bibinfo{person}{Bob McGrew}, \bibinfo{person}{Josh Tobin},
  \bibinfo{person}{OpenAI Pieter~Abbeel}, {and} \bibinfo{person}{Wojciech
  Zaremba}.} \bibinfo{year}{2017}\natexlab{}.
\newblock \showarticletitle{Hindsight Experience Replay}. In
  \bibinfo{booktitle}{\emph{NeurIPS}}. \bibinfo{publisher}{Curran Associates}.
\newblock


\bibitem[\protect\citeauthoryear{Bain and Sammut}{Bain and Sammut}{1999}]%
        {bain1996}
\bibfield{author}{\bibinfo{person}{Michael Bain} {and} \bibinfo{person}{Claude
  Sammut}.} \bibinfo{year}{1999}\natexlab{}.
\newblock \showarticletitle{A Framework for Behavioural Cloning}.
\newblock \bibinfo{journal}{\emph{Machine Intelligence}}  \bibinfo{volume}{15}
  (\bibinfo{year}{1999}), \bibinfo{pages}{103--129}.
\newblock


\bibitem[\protect\citeauthoryear{Brown, Goo, Nagarajan, and Niekum}{Brown
  et~al\mbox{.}}{2019b}]%
        {Brown2019ExtrapolatingBS}
\bibfield{author}{\bibinfo{person}{Daniel~S. Brown}, \bibinfo{person}{Wonjoon
  Goo}, \bibinfo{person}{P. Nagarajan}, {and} \bibinfo{person}{S. Niekum}.}
  \bibinfo{year}{2019}\natexlab{b}.
\newblock \showarticletitle{Extrapolating Beyond Suboptimal Demonstrations via
  Inverse Reinforcement Learning from Observations}. In
  \bibinfo{booktitle}{\emph{ICML}}. \bibinfo{publisher}{PMLR}.
\newblock


\bibitem[\protect\citeauthoryear{Brown, Goo, and Niekum}{Brown
  et~al\mbox{.}}{2019a}]%
        {Brown2019BetterthanDemonstratorIL}
\bibfield{author}{\bibinfo{person}{Daniel~S. Brown}, \bibinfo{person}{Wonjoon
  Goo}, {and} \bibinfo{person}{S. Niekum}.} \bibinfo{year}{2019}\natexlab{a}.
\newblock \showarticletitle{Better-than-Demonstrator Imitation Learning via
  Automatically-Ranked Demonstrations}. In \bibinfo{booktitle}{\emph{CoRL}}.
  \bibinfo{publisher}{PMLR}.
\newblock


\bibitem[\protect\citeauthoryear{Brys, Harutyunyan, Suay, Chernova, Taylor, and
  Now{\'e}}{Brys et~al\mbox{.}}{2015}]%
        {Brys2015ReinforcementLF}
\bibfield{author}{\bibinfo{person}{T. Brys}, \bibinfo{person}{A. Harutyunyan},
  \bibinfo{person}{Halit~Bener Suay}, \bibinfo{person}{S. Chernova},
  \bibinfo{person}{Matthew~E. Taylor}, {and} \bibinfo{person}{A. Now{\'e}}.}
  \bibinfo{year}{2015}\natexlab{}.
\newblock \showarticletitle{Reinforcement Learning from Demonstration through
  Shaping}. In \bibinfo{booktitle}{\emph{IJCAI}}. \bibinfo{publisher}{IJCAI}.
\newblock


\bibitem[\protect\citeauthoryear{Chen, Paleja, and Gombolay}{Chen
  et~al\mbox{.}}{2020}]%
        {Chen2020LearningFS}
\bibfield{author}{\bibinfo{person}{Letian Chen}, \bibinfo{person}{Rohan~R.
  Paleja}, {and} \bibinfo{person}{M. Gombolay}.}
  \bibinfo{year}{2020}\natexlab{}.
\newblock \showarticletitle{Learning from Suboptimal Demonstration via
  Self-Supervised Reward Regression}. In \bibinfo{booktitle}{\emph{CoRL}}.
  \bibinfo{publisher}{PMLR}.
\newblock


\bibitem[\protect\citeauthoryear{Devlin and Kudenko}{Devlin and
  Kudenko}{2012}]%
        {devlin12dynamicpbrs}
\bibfield{author}{\bibinfo{person}{Sam Devlin} {and} \bibinfo{person}{D.
  Kudenko}.} \bibinfo{year}{2012}\natexlab{}.
\newblock \showarticletitle{Dynamic potential-based reward shaping}. In
  \bibinfo{booktitle}{\emph{AAMAS}}. \bibinfo{publisher}{IFAAMAS}.
\newblock


\bibitem[\protect\citeauthoryear{Ding, Florensa, Abbeel, and Phielipp}{Ding
  et~al\mbox{.}}{2019}]%
        {ding19gcil}
\bibfield{author}{\bibinfo{person}{Yiming Ding}, \bibinfo{person}{Carlos
  Florensa}, \bibinfo{person}{Pieter Abbeel}, {and} \bibinfo{person}{Mariano
  Phielipp}.} \bibinfo{year}{2019}\natexlab{}.
\newblock \showarticletitle{Goal-conditioned Imitation Learning}. In
  \bibinfo{booktitle}{\emph{NeurIPS}}. \bibinfo{publisher}{Curran Associates}.
\newblock


\bibitem[\protect\citeauthoryear{Ecoffet, Huizinga, Lehman, Stanley, and
  Clune}{Ecoffet et~al\mbox{.}}{2021}]%
        {Ecoffet2019GoExploreAN}
\bibfield{author}{\bibinfo{person}{Adrien Ecoffet}, \bibinfo{person}{Joost
  Huizinga}, \bibinfo{person}{J. Lehman}, \bibinfo{person}{Kenneth~O. Stanley},
  {and} \bibinfo{person}{J. Clune}.} \bibinfo{year}{2021}\natexlab{}.
\newblock \showarticletitle{First return, then explore}.
\newblock \bibinfo{journal}{\emph{Nature}}  \bibinfo{volume}{590}
  (\bibinfo{year}{2021}), \bibinfo{pages}{580--586}.
\newblock


\bibitem[\protect\citeauthoryear{Ghasemipour, Zemel, and Gu}{Ghasemipour
  et~al\mbox{.}}{2019}]%
        {Ghasemipour2019ADM}
\bibfield{author}{\bibinfo{person}{Seyed Kamyar~Seyed Ghasemipour},
  \bibinfo{person}{R. Zemel}, {and} \bibinfo{person}{S. Gu}.}
  \bibinfo{year}{2019}\natexlab{}.
\newblock \showarticletitle{A Divergence Minimization Perspective on Imitation
  Learning Methods}. In \bibinfo{booktitle}{\emph{CoRL}}.
  \bibinfo{publisher}{PMLR}.
\newblock


\bibitem[\protect\citeauthoryear{Grzes}{Grzes}{2017}]%
        {Grzes2017RewardSI}
\bibfield{author}{\bibinfo{person}{M. Grzes}.} \bibinfo{year}{2017}\natexlab{}.
\newblock \showarticletitle{Reward Shaping in Episodic Reinforcement Learning}.
  In \bibinfo{booktitle}{\emph{AAMAS}}. \bibinfo{publisher}{IFAAMAS}.
\newblock


\bibitem[\protect\citeauthoryear{Grzes and Kudenko}{Grzes and Kudenko}{2010}]%
        {Grzes2010OnlineLO}
\bibfield{author}{\bibinfo{person}{Marek Grzes} {and} \bibinfo{person}{Daniel
  Kudenko}.} \bibinfo{year}{2010}\natexlab{}.
\newblock \showarticletitle{Online learning of shaping rewards in reinforcement
  learning}.
\newblock \bibinfo{journal}{\emph{Neural networks : the official journal of the
  International Neural Network Society}}  \bibinfo{volume}{23 4}
  (\bibinfo{year}{2010}), \bibinfo{pages}{541--50}.
\newblock


\bibitem[\protect\citeauthoryear{Guo, Chang, Yu, Tesauro, and Campbell}{Guo
  et~al\mbox{.}}{2019}]%
        {Guo2019HybridRL}
\bibfield{author}{\bibinfo{person}{Xiaoxiao Guo}, \bibinfo{person}{Shiyu
  Chang}, \bibinfo{person}{Mo Yu}, \bibinfo{person}{G. Tesauro}, {and}
  \bibinfo{person}{Murray Campbell}.} \bibinfo{year}{2019}\natexlab{}.
\newblock \showarticletitle{Hybrid Reinforcement Learning with Expert State
  Sequences}. In \bibinfo{booktitle}{\emph{AAAI}}. \bibinfo{publisher}{AAAI
  Press}.
\newblock


\bibitem[\protect\citeauthoryear{Hansen and Ostermeier}{Hansen and
  Ostermeier}{2001}]%
        {cmaes2001}
\bibfield{author}{\bibinfo{person}{Nikolaus Hansen} {and}
  \bibinfo{person}{Andreas Ostermeier}.} \bibinfo{year}{2001}\natexlab{}.
\newblock \showarticletitle{Completely Derandomized Self-Adaptation in
  Evolution Strategies}.
\newblock \bibinfo{journal}{\emph{Evolutionary Computation}}
  \bibinfo{volume}{9}, \bibinfo{number}{2} (\bibinfo{year}{2001}),
  \bibinfo{pages}{159–195}.
\newblock


\bibitem[\protect\citeauthoryear{Hester, Vecer{\'i}k, Pietquin, Lanctot,
  Schaul, Piot, Horgan, Quan, Sendonaris, Osband, Dulac-Arnold, Agapiou, Leibo,
  and Gruslys}{Hester et~al\mbox{.}}{2018}]%
        {Hester2018DeepQF}
\bibfield{author}{\bibinfo{person}{Todd Hester}, \bibinfo{person}{Matej
  Vecer{\'i}k}, \bibinfo{person}{O. Pietquin}, \bibinfo{person}{Marc Lanctot},
  \bibinfo{person}{T. Schaul}, \bibinfo{person}{Bilal Piot},
  \bibinfo{person}{Dan Horgan}, \bibinfo{person}{John Quan},
  \bibinfo{person}{A. Sendonaris}, \bibinfo{person}{Ian Osband},
  \bibinfo{person}{Gabriel Dulac-Arnold}, \bibinfo{person}{J. Agapiou},
  \bibinfo{person}{Joel~Z. Leibo}, {and} \bibinfo{person}{A. Gruslys}.}
  \bibinfo{year}{2018}\natexlab{}.
\newblock \showarticletitle{Deep Q-learning From Demonstrations}. In
  \bibinfo{booktitle}{\emph{AAAI}}. \bibinfo{publisher}{AAAI Press}.
\newblock


\bibitem[\protect\citeauthoryear{Ho and Ermon}{Ho and Ermon}{2016}]%
        {gail2016}
\bibfield{author}{\bibinfo{person}{Jonathan Ho} {and} \bibinfo{person}{Stefano
  Ermon}.} \bibinfo{year}{2016}\natexlab{}.
\newblock \showarticletitle{Generative Adversarial Imitation Learning}. In
  \bibinfo{booktitle}{\emph{NeurIPS}}. \bibinfo{publisher}{Curran Associates}.
\newblock


\bibitem[\protect\citeauthoryear{Kaelbling}{Kaelbling}{1993}]%
        {kaelbling93learningto}
\bibfield{author}{\bibinfo{person}{Leslie~Pack Kaelbling}.}
  \bibinfo{year}{1993}\natexlab{}.
\newblock \showarticletitle{Learning to Achieve Goals}. In
  \bibinfo{booktitle}{\emph{IJCAI}}. \bibinfo{publisher}{IJCAI}.
\newblock


\bibitem[\protect\citeauthoryear{Khansari-Zadeh and Khatib}{Khansari-Zadeh and
  Khatib}{2017}]%
        {KhansariZadeh2017LearningPF}
\bibfield{author}{\bibinfo{person}{Seyed~Mohammad Khansari-Zadeh} {and}
  \bibinfo{person}{Oussama Khatib}.} \bibinfo{year}{2017}\natexlab{}.
\newblock \showarticletitle{Learning potential functions from human
  demonstrations with encapsulated dynamic and compliant behaviors}.
\newblock \bibinfo{journal}{\emph{Autonomous Robots}}  \bibinfo{volume}{41}
  (\bibinfo{year}{2017}), \bibinfo{pages}{45--69}.
\newblock


\bibitem[\protect\citeauthoryear{Liu, Gupta, Abbeel, and Levine}{Liu
  et~al\mbox{.}}{2018}]%
        {liuabishek2017}
\bibfield{author}{\bibinfo{person}{Yuxuan Liu}, \bibinfo{person}{Abhishek
  Gupta}, \bibinfo{person}{Pieter Abbeel}, {and} \bibinfo{person}{Sergey
  Levine}.} \bibinfo{year}{2018}\natexlab{}.
\newblock \showarticletitle{Imitation from Observation: Learning to Imitate
  Behaviors from Raw Video via Context Translation}. In
  \bibinfo{booktitle}{\emph{ICRA}}. \bibinfo{publisher}{IEEE}.
\newblock


\bibitem[\protect\citeauthoryear{Nair, McGrew, Andrychowicz, Zaremba, and
  Abbeel}{Nair et~al\mbox{.}}{2018}]%
        {Nair2018OvercomingEI}
\bibfield{author}{\bibinfo{person}{Ashvin Nair}, \bibinfo{person}{Bob McGrew},
  \bibinfo{person}{Marcin Andrychowicz}, \bibinfo{person}{Wojciech Zaremba},
  {and} \bibinfo{person}{P. Abbeel}.} \bibinfo{year}{2018}\natexlab{}.
\newblock \showarticletitle{Overcoming Exploration in Reinforcement Learning
  with Demonstrations}. In \bibinfo{booktitle}{\emph{ICRA}}.
  \bibinfo{publisher}{IEEE}.
\newblock


\bibitem[\protect\citeauthoryear{Ng, Harada, and Russell}{Ng
  et~al\mbox{.}}{1999}]%
        {Ng99policyinvariance}
\bibfield{author}{\bibinfo{person}{Andrew~Y. Ng}, \bibinfo{person}{Daishi
  Harada}, {and} \bibinfo{person}{Stuart Russell}.}
  \bibinfo{year}{1999}\natexlab{}.
\newblock \showarticletitle{Policy invariance under reward transformations:
  Theory and application to reward shaping}. In
  \bibinfo{booktitle}{\emph{ICML}}. \bibinfo{publisher}{PMLR}.
\newblock


\bibitem[\protect\citeauthoryear{Paine, Colmenarejo, Wang, Reed, Aytar, Pfaff,
  Hoffman, Barth-Maron, Cabi, Budden, and Freitas}{Paine et~al\mbox{.}}{2018}]%
        {Paine2018OneShotHI}
\bibfield{author}{\bibinfo{person}{T. Paine}, \bibinfo{person}{Sergio~Gomez
  Colmenarejo}, \bibinfo{person}{Ziyu Wang}, \bibinfo{person}{Scott~E. Reed},
  \bibinfo{person}{Yusuf Aytar}, \bibinfo{person}{T. Pfaff},
  \bibinfo{person}{Matthew~W. Hoffman}, \bibinfo{person}{Gabriel Barth-Maron},
  \bibinfo{person}{Serkan Cabi}, \bibinfo{person}{D. Budden}, {and}
  \bibinfo{person}{N.~D. Freitas}.} \bibinfo{year}{2018}\natexlab{}.
\newblock \showarticletitle{One-Shot High-Fidelity Imitation: Training
  Large-Scale Deep Nets with RL}.
\newblock \bibinfo{journal}{\emph{ArXiv}}  \bibinfo{volume}{abs/1810.05017}.
\newblock


\bibitem[\protect\citeauthoryear{Pardo, Tavakoli, Levdik, and Kormushev}{Pardo
  et~al\mbox{.}}{2018}]%
        {Pardo2018TimeLI}
\bibfield{author}{\bibinfo{person}{Fabio Pardo}, \bibinfo{person}{Arash
  Tavakoli}, \bibinfo{person}{Vitaly Levdik}, {and} \bibinfo{person}{Petar
  Kormushev}.} \bibinfo{year}{2018}\natexlab{}.
\newblock \showarticletitle{Time Limits in Reinforcement Learning}. In
  \bibinfo{booktitle}{\emph{ICML}}. \bibinfo{publisher}{PMLR}.
\newblock


\bibitem[\protect\citeauthoryear{Paul, Vanbaar, and Roy-Chowdhury}{Paul
  et~al\mbox{.}}{2019}]%
        {Paul2019LearningFT}
\bibfield{author}{\bibinfo{person}{Sujoy Paul}, \bibinfo{person}{Jeroen
  Vanbaar}, {and} \bibinfo{person}{Amit Roy-Chowdhury}.}
  \bibinfo{year}{2019}\natexlab{}.
\newblock \showarticletitle{Learning from Trajectories via Subgoal Discovery}.
  In \bibinfo{booktitle}{\emph{NeurIPS}}. \bibinfo{publisher}{Curran
  Associates}.
\newblock


\bibitem[\protect\citeauthoryear{Pavse, Torabi, Hanna, Warnell, and
  Stone}{Pavse et~al\mbox{.}}{2020}]%
        {ridm2020}
\bibfield{author}{\bibinfo{person}{Brahma~S. Pavse}, \bibinfo{person}{Faraz
  Torabi}, \bibinfo{person}{Josiah Hanna}, \bibinfo{person}{Garrett Warnell},
  {and} \bibinfo{person}{Peter Stone}.} \bibinfo{year}{2020}\natexlab{}.
\newblock \showarticletitle{RIDM: Reinforced Inverse Dynamics Modeling for
  Learning from a Single Observed Demonstration}. In
  \bibinfo{booktitle}{\emph{IROS}}. \bibinfo{publisher}{IEEE}.
\newblock


\bibitem[\protect\citeauthoryear{Salimans and Chen}{Salimans and Chen}{2018}]%
        {salimans2018LearningMR}
\bibfield{author}{\bibinfo{person}{Tim Salimans} {and}
  \bibinfo{person}{Richard~J. Chen}.} \bibinfo{year}{2018}\natexlab{}.
\newblock \showarticletitle{Learning Montezuma's Revenge from a Single
  Demonstration}. In \bibinfo{booktitle}{\emph{Workshop on Deep Reinforcement
  learning at NeurIPS}}.
\newblock


\bibitem[\protect\citeauthoryear{Schaul, Horgan, Gregor, and Silver}{Schaul
  et~al\mbox{.}}{2015}]%
        {Schaul2015UniversalVF}
\bibfield{author}{\bibinfo{person}{Tom Schaul}, \bibinfo{person}{Dan Horgan},
  \bibinfo{person}{K. Gregor}, {and} \bibinfo{person}{D. Silver}.}
  \bibinfo{year}{2015}\natexlab{}.
\newblock \showarticletitle{Universal Value Function Approximators}. In
  \bibinfo{booktitle}{\emph{ICML}}. \bibinfo{publisher}{PMLR}.
\newblock


\bibitem[\protect\citeauthoryear{Schubert, Oguz, and Toussaint}{Schubert
  et~al\mbox{.}}{2021}]%
        {schubert21planbased}
\bibfield{author}{\bibinfo{person}{Ingmar Schubert}, \bibinfo{person}{Ozgur
  Oguz}, {and} \bibinfo{person}{Mark Toussaint}.}
  \bibinfo{year}{2021}\natexlab{}.
\newblock \showarticletitle{Plan-Based Relaxed Reward Shaping for Goal-Directed
  Tasks}. In \bibinfo{booktitle}{\emph{ICLR}}.
\newblock


\bibitem[\protect\citeauthoryear{Sermanet, Lynch, Chebotar, Hsu, Jang, Schaal,
  and Levine}{Sermanet et~al\mbox{.}}{2018}]%
        {Sermanet2018TimeContrastiveNS}
\bibfield{author}{\bibinfo{person}{Pierre Sermanet}, \bibinfo{person}{Corey
  Lynch}, \bibinfo{person}{Yevgen Chebotar}, \bibinfo{person}{Jasmine Hsu},
  \bibinfo{person}{Eric Jang}, \bibinfo{person}{S. Schaal}, {and}
  \bibinfo{person}{Sergey Levine}.} \bibinfo{year}{2018}\natexlab{}.
\newblock \showarticletitle{Time-Contrastive Networks: Self-Supervised Learning
  from Video}. In \bibinfo{booktitle}{\emph{ICRA}}. \bibinfo{publisher}{IEEE}.
\newblock


\bibitem[\protect\citeauthoryear{Suay, Brys, Taylor, and Chernova}{Suay
  et~al\mbox{.}}{2016}]%
        {Suay2016LearningFD}
\bibfield{author}{\bibinfo{person}{Halit~Bener Suay}, \bibinfo{person}{T.
  Brys}, \bibinfo{person}{Matthew~E. Taylor}, {and} \bibinfo{person}{S.
  Chernova}.} \bibinfo{year}{2016}\natexlab{}.
\newblock \showarticletitle{Learning from Demonstration for Shaping through
  Inverse Reinforcement Learning}. In \bibinfo{booktitle}{\emph{AAMAS}}.
  \bibinfo{publisher}{IFAAMAS}.
\newblock


\bibitem[\protect\citeauthoryear{Taylor, Suay, and Chernova}{Taylor
  et~al\mbox{.}}{2011}]%
        {Taylor2011IntegratingRL}
\bibfield{author}{\bibinfo{person}{Matthew~E. Taylor},
  \bibinfo{person}{Halit~Bener Suay}, {and} \bibinfo{person}{S. Chernova}.}
  \bibinfo{year}{2011}\natexlab{}.
\newblock \showarticletitle{Integrating reinforcement learning with human
  demonstrations of varying ability}. In \bibinfo{booktitle}{\emph{AAMAS}}.
  \bibinfo{publisher}{IFAAMAS}.
\newblock


\bibitem[\protect\citeauthoryear{Torabi, Warnell, and Stone}{Torabi
  et~al\mbox{.}}{2018}]%
        {bco2018}
\bibfield{author}{\bibinfo{person}{Faraz Torabi}, \bibinfo{person}{Garrett
  Warnell}, {and} \bibinfo{person}{Peter Stone}.}
  \bibinfo{year}{2018}\natexlab{}.
\newblock \showarticletitle{Behavioral Cloning from Observation}. In
  \bibinfo{booktitle}{\emph{IJCAI}}. \bibinfo{publisher}{IJCAI}.
\newblock


\bibitem[\protect\citeauthoryear{Torabi, Warnell, and Stone}{Torabi
  et~al\mbox{.}}{2019a}]%
        {gaifo2019}
\bibfield{author}{\bibinfo{person}{Faraz Torabi}, \bibinfo{person}{Garrett
  Warnell}, {and} \bibinfo{person}{Peter Stone}.}
  \bibinfo{year}{2019}\natexlab{a}.
\newblock \showarticletitle{Generative Adversarial Imitation from Observation}.
  In \bibinfo{booktitle}{\emph{Imitation, Intent, and Interaction (I3) Workshop
  at ICML}}. \bibinfo{publisher}{PMLR}.
\newblock


\bibitem[\protect\citeauthoryear{Torabi, Warnell, and Stone}{Torabi
  et~al\mbox{.}}{2019b}]%
        {torabi2019}
\bibfield{author}{\bibinfo{person}{Faraz Torabi}, \bibinfo{person}{Garrett
  Warnell}, {and} \bibinfo{person}{Peter Stone}.}
  \bibinfo{year}{2019}\natexlab{b}.
\newblock \showarticletitle{Recent Advances in Imitation Learning from
  Observation}. In \bibinfo{booktitle}{\emph{IJCAI}}.
  \bibinfo{publisher}{IJCAI}.
\newblock


\bibitem[\protect\citeauthoryear{Watkins}{Watkins}{1989}]%
        {qlearningwatkins1989}
\bibfield{author}{\bibinfo{person}{Christopher John Cornish~Hellaby Watkins}.}
  \bibinfo{year}{1989}\natexlab{}.
\newblock \emph{\bibinfo{title}{Learning from delayed rewards}}.
\newblock \bibinfo{thesistype}{Ph.D. Dissertation}. \bibinfo{school}{University
  of Cambridge, England}.
\newblock


\bibitem[\protect\citeauthoryear{Wu, Mozifian, and Shkurti}{Wu
  et~al\mbox{.}}{2021}]%
        {Wu2020ShapingRewards}
\bibfield{author}{\bibinfo{person}{Yuchen Wu}, \bibinfo{person}{Melissa
  Mozifian}, {and} \bibinfo{person}{Florian Shkurti}.}
  \bibinfo{year}{2021}\natexlab{}.
\newblock \showarticletitle{Shaping Rewards for Reinforcement Learning with
  Imperfect Demonstrations using Generative Models}. In
  \bibinfo{booktitle}{\emph{ICRA}}. \bibinfo{publisher}{IEEE}.
\newblock


\bibitem[\protect\citeauthoryear{Wu, Charoenphakdee, Bao, Tangkaratt, and
  Sugiyama}{Wu et~al\mbox{.}}{2019}]%
        {Wu2019ImitationLF}
\bibfield{author}{\bibinfo{person}{Yueh-Hua Wu}, \bibinfo{person}{Nontawat
  Charoenphakdee}, \bibinfo{person}{Han Bao}, \bibinfo{person}{Voot
  Tangkaratt}, {and} \bibinfo{person}{Masashi Sugiyama}.}
  \bibinfo{year}{2019}\natexlab{}.
\newblock \showarticletitle{Imitation Learning from Imperfect Demonstration}.
  In \bibinfo{booktitle}{\emph{ICML}}. \bibinfo{publisher}{PMLR}.
\newblock


\bibitem[\protect\citeauthoryear{Zhu, Wang, Merel, Rusu, Erez, Cabi,
  Tunyasuvunakool, Kram{\'a}r, Hadsell, de~Freitas, and Heess}{Zhu
  et~al\mbox{.}}{2018}]%
        {Zhu-RSS-18}
\bibfield{author}{\bibinfo{person}{Yuke Zhu}, \bibinfo{person}{Ziyu Wang},
  \bibinfo{person}{Josh Merel}, \bibinfo{person}{Andrei Rusu},
  \bibinfo{person}{Tom Erez}, \bibinfo{person}{Serkan Cabi},
  \bibinfo{person}{Saran Tunyasuvunakool}, \bibinfo{person}{J{\'a}nos
  Kram{\'a}r}, \bibinfo{person}{Raia Hadsell}, \bibinfo{person}{Nando de
  Freitas}, {and} \bibinfo{person}{Nicolas Heess}.}
  \bibinfo{year}{2018}\natexlab{}.
\newblock \showarticletitle{Reinforcement and Imitation Learning for Diverse
  Visuomotor Skills}. In \bibinfo{booktitle}{\emph{RSS}}.
  \bibinfo{publisher}{RSS}.
\newblock


\bibitem[\protect\citeauthoryear{Zolna, Rostamzadeh, Bengio, Ahn, and
  Pinheiro}{Zolna et~al\mbox{.}}{2019}]%
        {Zolna2019ReinforcedII}
\bibfield{author}{\bibinfo{person}{Konrad Zolna}, \bibinfo{person}{N.
  Rostamzadeh}, \bibinfo{person}{Yoshua Bengio}, \bibinfo{person}{Sungjin Ahn},
  {and} \bibinfo{person}{Pedro H.~O. Pinheiro}.}
  \bibinfo{year}{2019}\natexlab{}.
\newblock \showarticletitle{Reinforced Imitation in Heterogeneous Action
  Space}. In \bibinfo{booktitle}{\emph{Imitation Learning and its Challenges in
  Robotics Workshop at NeurIPS}}.
\newblock


\end{thebibliography}

\clearpage
\newpage

\section*{Supplement}
\label{section:appendix}

\subsection{Proofs}
\label{subsection:proofs}

We consider the finite-horizon MDP $M = (S, A, P, r, \gamma)$ with horizon $H$ and the goal-conditioned version of $M$, $M^\dagger = ( S \times G, A, P^\dagger, r^\dagger, \gamma)$. All terms are defined as in Sections 2 and 3. We prove two supporting propositions before proving the main result (Theorem 1).

\textbf{Proposition 1. } Let $\pi(\cdot \mid s)$ be a policy defined in $M$, and let $f(\pi)(\cdot \mid [s,  g]) $ be the extension of $\pi(\cdot \mid s)$ to $M^\dagger$. Then $V_{\pi}(s_t) = V_{f(\pi)}([s_t,  g_t]) \:\: \forall s_t \in S, g_t \in G$, and $t \in \{1,\dots H\}$.

\begin{proof}
Let $s^g_t$ denote a state-goal pair, $[s_t,  g_t]$. 
Let $\tau_t$ denote trajectories in $M$, i.e. $\tau_t = (s_t, a_t, \dots, s_H, a_H)$. 
Let $\tau^G_t$ denote (goal-conditioned) trajectories in $M^\dagger$, i.e. $(s^g_t, a_t, \dots , s^g_H, a_H )$. The conditional probability of a trajectory is 
\begin{equation*}
p_{\pi}(\tau_t \mid s_t, a_t) = \prod_{t'=t}^{H-1} \pi(a_{t'} \mid s_{t'}) P(s_{t'+1}\mid s_{t'}, a_{t'} ).
\end{equation*}
The reward of a trajectory is defined as follows: 
\begin{equation*}
r(\tau_t) = \sum_{t'= t}^{H-1} \gamma^{t'-t} r(s_{t'+1}, a_{t'}, s_{t'}).    
\end{equation*}
Notice that as the reward function does not depend on goals, $r^\dagger(\tau^G_t) = r(\tau_t)$.

Computing the value of $f(\pi)$ at $s^g_t$ from definition:
\begin{align*}
    V_{f(\pi)}(s^g_t) &= \E_{\tau^G_t \sim p_{f(\pi)}} \Big[r^\dagger(\tau^G_t) \mid s^g_t \Big] \\ 
                         &= \sum_{\tau^G_t}  \Big[ p_{f(\pi)}(\tau^G_t \mid s^g_t) r^\dagger(\tau^G_t) \Big] \\
                         &= \sum_{\tau^G_t}  \Big[ \Big( \prod_{t'=t}^{H-1} f(\pi)(a_{t'} \mid s^g_{t'}) P^\dagger(s^g_{t'+1}\mid s^g_{t'},     a_{t'} ) \Big) r^\dagger(\tau^G_t) \Big] \\
                         &= \sum_{\tau^G_t}  \Big[ \Big( \prod_{t'=t}^{H-1} \pi(a_{t'} \mid s_{t'}) P(s_{t'+1}\mid s_{t'}, a_{t'}) \\
                         & \qquad \qquad  P(g_{t'+1} \mid [s_{t'}, g_{t'}], a_{t'})\Big) r(\tau_t) \Big].
\end{align*}

The outermost summation is over $\tau^G_t$, representing summing over all values of $(s^g_t, a_t, \dots , s^g_H, a_H )$ --- a summation over all possible trajectories in $M^\dagger$. We organize these trajectories into sets where the variables $\{s_{t'}, a_{t'}\}_{t'=t}^{H-1}$ are fixed, and only $\{g_{t'}\}_{t'=t}^{H-1}$ may vary. Let $\tau^{s,a}_t \coloneqq \{s_{t'}, a_{t'}\}_{t'=t}^{H-1}$ (state-action sequences) and $\tau^{g \mid sa}_t \coloneqq \{g_{t'} \mid s_{t'}, a_{t'}\}_{t'=t}^{H-1}$ (sequences of goals). Thus, we split the outermost summation over $\tau^G_t$ into two summations where the outer summation is over  $\tau^{s,a}_t $  and the inner summation is over $\tau^{g \mid sa}_t$. Reorganizing the sum is permissible by the commutative property of addition and the fact that we are dealing with finite sums. Since the trajectory rewards, transition probabilities, and action probabilities only depend on $\tau^{s,a}_t $, we may factor it out of the inner summation. Continuing the derivation from above:
\begin{align*}
    V_{f(\pi)}(s^g_t) &= \sum_{\tau^{s,a}_t} \sum_{\tau^{g \mid sa}_t} \Big[ \Big( \prod_{t'=t}^{H-1} \pi(a_{t'} \mid s_{t'}) P(s_{t'+1}\mid s_{t'}, a_{t'}) \\
                    & \qquad \qquad  P(g_{t'+1} \mid [s_{t'}, g_{t'}], a_{t'})\Big) r(\tau_t)  \Big] \\
                    &= \sum_{\tau^{s,a}_t} \Big(\Big( \prod_{t'=t}^{H-1} \pi(a_{t'} \mid s_{t'}) P(s_{t'+1}\mid s_{t'}, a_{t'})\Big) r(\tau_t) \Big)  \\
                    & \qquad \qquad \sum_{\tau^{g\mid sa}_t} \Big( \prod_{t'=t}^{H-1} P(g_{t'+1} \mid [s_{t'}, g_{t'}], a_{t'}) \Big).
\end{align*}

In the innermost summation over $\tau^{g\mid sa}_t$, we are left only with the product of conditional goal probabilities. This is precisely the probability of a goal sequence $p(\tau^{g \mid sa}_t)$, given that the states and actions are fixed, which sums to 1. The remaining terms depend only on the sequence of states and actions, and are precisely $V_\pi(s_t)$, completing the proof.
\begin{align*}
    V_{f(\pi)}(s^g_t) &=  \sum_{\tau^{s,a}_t} \Big( \Big( \prod_{t'=t}^{H-1} \pi(a_{t'} \mid s_{t'}) P(s_{t'+1}\mid s_{t'}, a_{t'}) \Big) r(\tau_t) \Big)  \\
                    & \qquad \qquad  \sum_{\tau^{g\mid sa}_t} p(\tau^{g\mid sa}_t) \\ 
                    &= \sum_{\tau^{s,a}_t} \Big( \Big( \prod_{t'=t}^{H-1} \pi(a_{t'} \mid s_{t'}) P(s_{t'+1}\mid s_{t'}, a_{t'})\Big) r(\tau_t) \Big) \\
                    &= V_{\pi}(s_t).
\end{align*}
\end{proof}

\textbf{Corollary 1.1. } $Q_{f(\pi)}([s_t, g_t], a_t) = Q_{\pi}(s_t, a_t)$ for all $s_t \in S, g_t \in G, a_t \in A,  t \in \{1,\dots H\}$.

\begin{proof}
The proof follows the same method as the proof for Proposition 1.
\end{proof}

\textbf{Proposition 2. } Let $\pi^*(\cdot \mid s)$ be an optimal policy in $M$, and let $f(\pi^*)(\cdot \mid [s,  g])$ be its extension to $M^\dagger$. Then $f(\pi^*)$ is an optimal policy in $M^\dagger$.

\begin{proof}
We will prove the statement by showing that the value of $f(\pi^*)$ is already maximal. 

By Corollary 1.1, $Q_{f(\pi^*)}([s, g], a) = Q_{\pi^*}(s, a)$. By definition of optimality, $\max_a Q_{\pi^*}(s,a)  = V_{\pi^*}(s)$. By Proposition 1, $V_{\pi^*}(s) = V_{f(\pi^*)}(s)$. Applying these three facts,
\begin{align*}
    \max_a Q_{f(\pi^*)}([s, g], a) &= \max_a Q_{\pi^*}(s,a) \\ 
             &= V_{\pi^*}(s) \\ 
             &= V_{f(\pi^*)}([s,  g]).
\end{align*}
Since $\max_a Q_{f(\pi^*)}([s, g], a) = V_{f(\pi^*)}([s,  g])$, the value of $f(\pi^*)([s,  g])$ is already maximal for all $[s,  g] \in S \times G$. Thus, $f(\pi^*)$ cannot be improved and must already be optimal in $M^\dagger$.

\end{proof}

For reader convenience, the statement and proof of Theorem 1 is reproduced below.

\textit{\textbf{Theorem 1. } Let $\pi^{\dagger*}(\cdot \mid [s, g])$ be an optimal policy in $M^\dagger$.  Define an arbitrary state-goal mapping $\Gamma: S \mapsto G_0 \subseteq G$, where $G_0$ is the set of goals in $M^\dagger$ that has positive probability of being reached by any policy. Then $\pi_\Gamma^{\dagger*}(\cdot \mid s) \coloneqq \pi^{\dagger*}(\cdot \mid [s, \Gamma(s)])$ is an optimal policy in $M$.
}
\begin{proof}
We prove the statement by considering the optimal state-action value functions in $M$ and $M^\dagger$. Denote the set of optimal actions at state-goal $[s,  g] \in M^\dagger$ by $A_{s, g} =\argmax_a Q^{\dagger*}([s, g], a) $.

Let $\pi^*(\cdot \mid s)$ be an optimal policy in $M$, and let $f(\pi^*)$ denote the policy induced by $\pi^*$ in $M^\dagger$ as in Proposition 2. By Proposition 2, $f(\pi^*)$ is an optimal policy in $M^\dagger$. Since the optimal value function in $M^\dagger$ is unique, $Q_{f(\pi^*)}([s,g],a) = Q^{\dagger*}([s,g],a)$ for all $s \in S$,  $g \in G$, and $a \in A$. By Corollary 1.1, $Q_{f(\pi^*)}([s,g],a) = Q_{\pi^*}(s,a)$. 
Thus,
\begin{align*}
    A_{s,g} &\coloneqq \argmax_a Q^{\dagger*}([s, g], a) \\
           &= \argmax_a Q_{f(\pi^*)}([s, g], a) \\ 
           &= \argmax_a Q_{\pi^*}(s,a).
\end{align*}
The computation shows that the set of optimal actions in $M^\dagger$ at some $[s,  g] \in S \times G$ is the same as the set of optimal actions in $M$ at $s$ --- that is, the set of optimal actions is independent of the goal $g$. Since the set of actions with positive probability under $\pi^{\dagger*}(\cdot \mid [s,  g])$ is a subset of $A_{s,g}$ for $g\in G_0$, we have $\pi_{\Gamma}^{\dagger *}(\cdot \mid s) := \pi^{\dagger *}(\cdot \mid [s, \Gamma(s)])\subset A_{s, \Gamma(s)} = \argmax_aQ_{\pi^*}(s,a)$ for all $s\in S$. Thus, $\pi^{\dagger*}_\Gamma$ is an optimal policy in $M$.

\end{proof}

We note that there is an alternative way to think about obtaining optimal policies in $M$ from an optimal policy in $M^{\dagger}$. We could prove, given that the rewards in $M^{\dagger}$ are agnostic to the goals, that an optimal policy in $M^{\dagger}$ is still optimal if we change the goal transition probabilities (while maintaining that the same $[s,  g]$ have positive support). Then we could modify $M^{\dagger}$ to an MDP that is equivalent to $M$ by just changing the the goal transition probabilities, and thus we would obtain an optimal policy in $M$.

\subsection{Further Experimental Results}
\label{subsection:exp_results}

Figure \ref{fig:ablation_with_ridm} reproduces the ablation study presented in the main paper, but \textit{with} RIDM as well. The remainder of this section provides further experiments supporting the results in the main paper. 

\begin{figure*}[htbp]
\centering
  \includegraphics[width=7.5in]{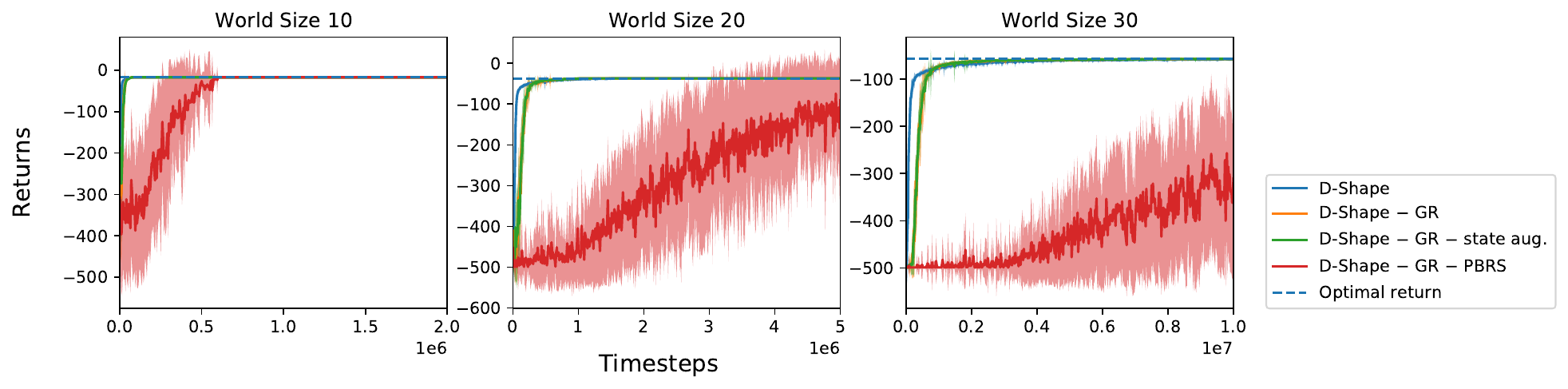}
  \caption{D-Shape compared to all three ablations of D-Shape using the optimal demonstration. D-Shape displays better sample-effiency than all ablations. The last ablation (DShape without PBRS or goal relabelling, i.e. state augmentation only) is shown on this plot. Optimal task return is shown as a horizontal line. Means and standard deviations are computed over 30 training runs.}
\label{fig:ablation_with_ridm}
\end{figure*}

\subsubsection{Stochastic Environments}
We also ran the main experiments on a stochastic version of the gridworld environment in the main paper, with similar findings. The results are shown in Figure \ref{fig:stochastic_core}. The main difference from the deterministic setting is that the learning curves are higher variance, demonstrating the validity of our approach even in stochastic environments. 

\begin{figure*}[th]
\centering
  \includegraphics[width=7.5in]{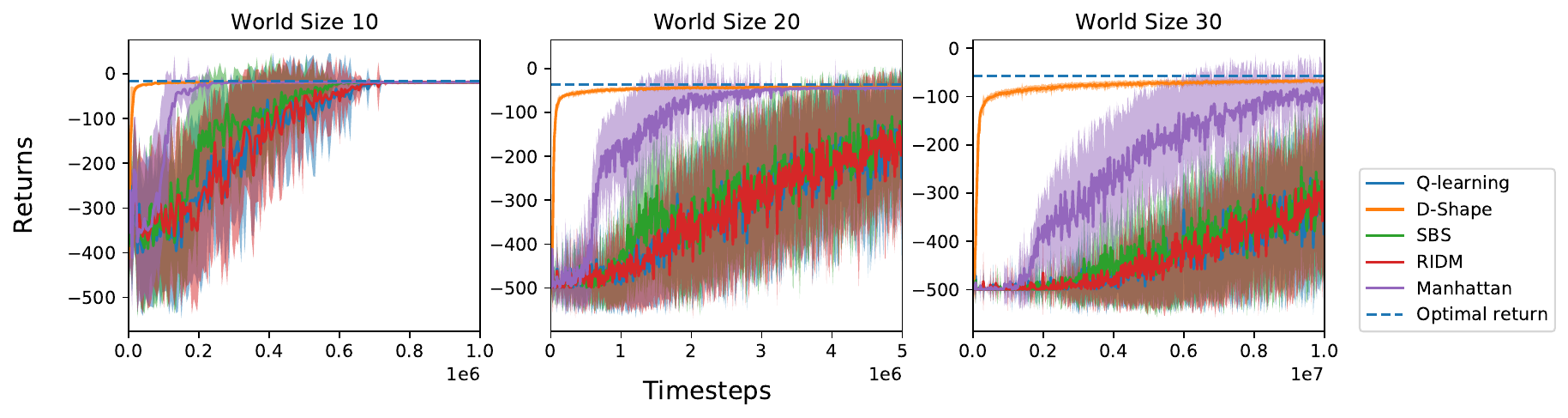}
  \caption{Learning curves of D-Shape with optimal quality demonstrations as compared to baseline methods, on a gridworld with stochastic transitions, where with 0.1 probability, a random action will be taken rather than that selected by the policy.}
\label{fig:stochastic_core}
\end{figure*}

\subsubsection{Further Types of Suboptimal Demonstrations}

The experiments in Section \ref{subsection:exp_subopt} show that for the first type of suboptimal demonstrations mentioned in the main paper (those that go to an incorrect goal state), D-Shape’s performance declines less than that of the Manhattan baseline. Further, DShape always converges to the optimal return, unlike Manhattan. Here, we examine whether D-Shape is robust to the second type of suboptimal demonstrations (those that do not decrease their distance to the goal at some timestep).

Two demonstration styles that satisfy this second type of suboptimality are considered here. 
Figure \ref{fig:subopt_demos_extended} shows how D-Shape handles  demonstrations consisting of an optimal path to the goal tile, where each state is repeated $N$ times ($N$ in $\{1, 3, 5\}$), as compared to the Manhattan baseline. Similar to Figure \ref{fig:vary_demos}, we find that D-Shape's convergence and sample efficiency is more robust to suboptimal demonstrations than Manhattan. 

We also consider demonstrations that are \textit{missing} states, an interesting type of suboptimality that may come up in real applications (Figure \ref{fig:subopt_demos_missing}). Interestingly, we find that D-Shape is strongly robust to this type of suboptimality, showing almost no decline in performance. We hypothesize this is because D-Shape treats each state as its own goal, so it is be able to handle demonstration trajectories with missing states. 

\begin{figure}[htbp]
\centering
  \includegraphics[width=3.5in]{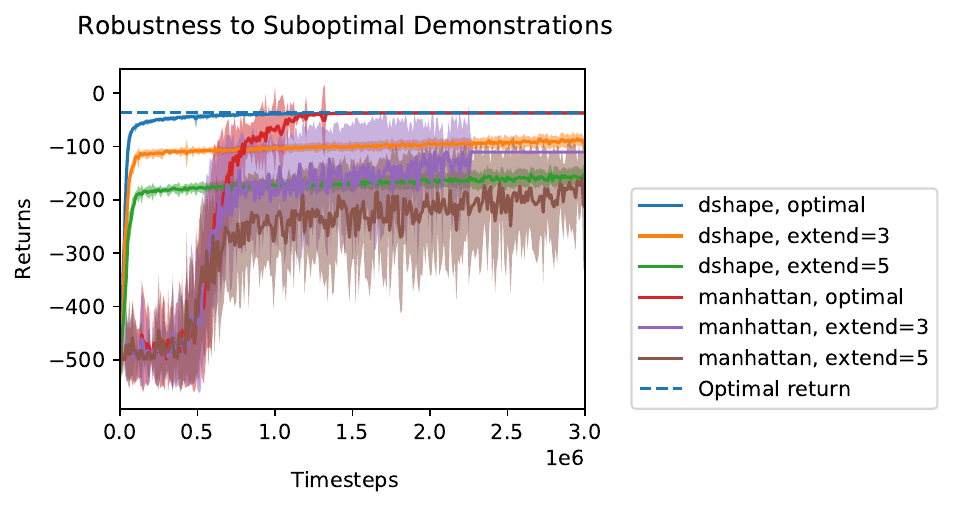}
  \caption{Learning curves of D-Shape compared to Manhattan with  suboptimal demonstrations that never increase distance from the goal, but ``linger" at each state for $N = \{1, 3, 5\}$ steps. Results shown are for the $20 \times 20$ gridworld.}
\label{fig:subopt_demos_extended}
\end{figure}

\begin{figure}[htbp]
\centering
  \includegraphics[width=3.5in]{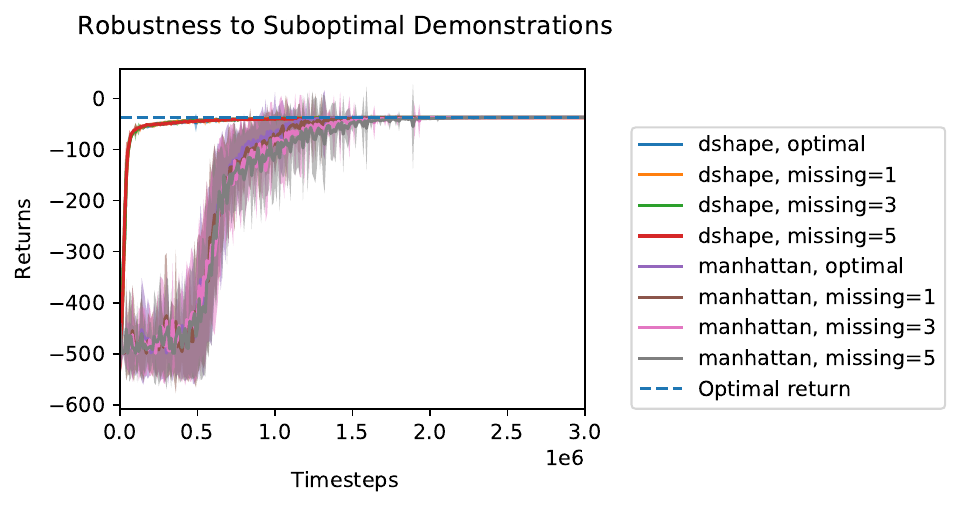}
  \caption{Learning curves of D-Shape compared to Manhattan with suboptimal demonstrations that are missing information (i.e. trajectories are missing states). Results shown are for the $20 \times 20$ gridworld.}
\label{fig:subopt_demos_missing}
\end{figure}

\subsection{Experimental Details}
\label{subsection:exp_procedure}

\subsubsection{Base RL Algorithm (Q-learning). } 
We use our own implementation of Q-learning with a tabular representation of the value function, where updates are performed based on data sampled from a replay buffer. Hyperparameters for baseline Q-learning are given in Table \ref{tbl:hyperparam}. Since all the methods in this paper use Q-learning as the base RL algorithm, these hyperparameters are also used there. 

\begin{table}[ht]
\centering
\begin{tabular}{ll}
\hline
Hyperparameter                  & Value  \\ \hline
training timesteps           & 250000 \\
max timesteps per episode & 500    \\
$\epsilon$               & 0.2    \\
$\alpha$                 & 0.1    \\
$\gamma$                 & 1      \\
updates per step      & 20     \\
buffer size           & 5000   \\ 
\hline
\\
\end{tabular}
\caption{Hyperparameters for the baseline of Q-learning. $\alpha$ is the learning rate, $\epsilon$ is the parameter for $\epsilon$-greedy exploration, and $\gamma$ is the discount factor.}
\label{tbl:hyperparam}
\end{table}

\subsubsection{PBRS in episodic settings. }
\citet{Ng99policyinvariance} proved their result in the continuing setting. \citet{Grzes2017RewardSI} showed that in order for the policy invariance result to hold in the episodic setting,  we must set $\phi(s_N) = 0$ for all terminal states $s_N$. We follow this recommendation for all methods that use a potential-based reward function (D-Shape, D-Shape ablations with potential-based shaping, and SBS). 

\subsubsection{D-Shape (our method).} Transitions in the replay buffer are of the form, $([s_t, g_t], a_t, r_t, [s_{t+1}, g_{t+1}])$. One assumption in the theory behind D-Shape is that $g_{t+1}$ depends only on $([s_t, g_t], a_t)$. Goals in the replay buffer come from two sources: the expert demonstration, or are sampled using the goal sampling strategy. We argue below that our implementation of D-Shape approximately satisfies the independence assumption. 

Goals either originate from the goal-sampling strategy or are expert states. D-Shape is implemented with a variant of the  \texttt{episode} goal sampling strategy \cite{andrychowicz17her}. The variant consists of uniformly sampling subsequent states at time $k$, $(s_k, s_{k+1})$ from the current episode to use as goal states at time $t$, $(g_t, g_{t+1})$.  Since $s_{k+1}$ depends on $s_k, a_k$, thus, $g_{t+1}$ depends only on $(g_t, a_k)$. In future work, we plan to investigate alternative goal sampling strategies that do not rely on the action $a_k$. The remaining transitions have time-aligned expert states as goal states, i.e. $(g_t, g_{t+1}) = (s^e_t, s^e_{t+1})$. Since our implementation includes timesteps are part of the state, $s^e_{t+1}$ can be determined from $s_t$. 

In the experiments, each sampled transition is relabelled with 3 goals. We found that relabelling with more goals did not improve performance. 

\subsubsection{RIDM \cite{ridm2020}. }
\citet{ridm2020} implemented RIDM with CMA-ES, an evolutionary optimization algorithm \cite{cmaes2001}. Since the optimization technique is not fundamental to their method, for fair comparison, we instead implement RIDM using the base RL algorithm (Q-learning). This turns out to be equivalent to Q-learning with state augmentation. RIDM does not introduce any hyperparameters.

\subsubsection{SBS \cite{Brys2015ReinforcementLF}. }
The original SBS method defines a potential function in terms of a similarity measure between states from the demonstration $s^d$, and the agent's current state $s$, $m(s^d, s)$. Denote the set of demonstrations by $D$. The potential function is of the following form: 

$$
\phi^D(s,a) = \max_{(s^d, a) \in D} m(s, s^d).
$$
This function computes the maximal similarity between agent states $s$ and demonstration states $s^d$ that come from samples with the same action as the agent. To implement SBS with state-only demonstrations, we simply remove the dependency on $a$:
$$
\phi^D(s) = \max_{s^d \in D} m(s, s^d).
$$ 

The particular similarity measure used by \citet{Brys2015ReinforcementLF} is a non-normalized multi-variate Gaussian with mean $s^d$ and covariance $\Sigma$, 
$$
g(s, s^d, \Sigma) = e^{(-\frac12 (s-s^d)^T \Sigma^{-1} (s - s^d))}.
$$
In \citet{Brys2015ReinforcementLF}, $\Sigma \coloneqq \sigma I$, where $\sigma$ is a constant, and state variables are mapped to $[0,1]$. Thus, in practice, the similarity metric is the exponentiated negative Euclidean distance with a scaling parameter $\sigma$.  Since the Euclidean distance metric is not a suitable distance metric for the discrete gridworld environments considered in our experiments, we instead replace it with the $L_1$ (Manhattan) distance. We also map the state variables into $[0, 1]$ and tune $\sigma$ for each gridworld task. As suggested by \citet{Brys2015ReinforcementLF}, demonstration states are stored in a k-d tree and $g$ is computed by querying the k-d tree. 

\textbf{Hyperparameter search procedure:} We perform a hyperparameter search over $\sigma$ and a reward scaling parameter for the shaping reward, $c$, using the $10\times10$ gridworld domain. The performance of each parameter combination is summarized by the sum of all returns along the training curve, an evaluation metric that balances between sample efficiency and asymptoptic return. 

The search range is given below, and the final hyperparameter choice is bolded:
$$\sigma \in [0.1, 1, \textbf{10}]$$
$$c \in [0.1, \textbf{1}, 10].$$

\subsubsection{Q-learning + Manhattan. }
This baseline consists of running baseline Q-learning with the following hybrid reward function:
$$
r_t^{task} - c || s_t - s^e_{t+1}||_1.
$$

The second term in the above reward function is the Manhattan distance of the current agent state from the next expert state, and is scaled by $c$. 

The hybrid reward function has three hyperparameters: $\alpha, \beta, \gamma \in [0, 1]$. We perform a hyperparameter search over the following values using the same procedure as for SBS. The final selected value is bolded: 
$$c \in [0.01, 0.1, \textbf{1}].$$

\end{document}